\documentclass[journal]{IEEEtran}

\usepackage{latexsym,bm}
\usepackage{xcolor}

\usepackage{autobreak}
\usepackage{makecell}
\usepackage[pdftex]{graphicx}
\usepackage{graphicx}
\usepackage{epstopdf}
\usepackage{stfloats}
\usepackage{epsfig}
\usepackage{booktabs}
\usepackage{multirow}
\usepackage[ruled,linesnumbered,noend]{algorithm2e}
\usepackage{amsmath}
\usepackage{amssymb}
\usepackage{amsthm}
\usepackage{epstopdf}
\usepackage{subfigure}
\usepackage{graphicx}
\usepackage{array}
\usepackage{comment} 
\usepackage{cite}

\newcommand{\Expect}{\mathbb{E}}

\IEEEoverridecommandlockouts

\ifCLASSINFOpdf
 
\else

\fi

\begin{document}

\title{Cost-Effective Federated Learning in\\

Mobile Edge Networks}

\author{Bing Luo,~\IEEEmembership{Member,~IEEE,}
        Xiang Li,~\IEEEmembership{Student~Member,~IEEE,}
        Shiqiang Wang,~\IEEEmembership{Member,~IEEE,}\\
        Jianwei Huang,~\IEEEmembership{Fellow,~IEEE,}
        Leandros Tassiulas,~\IEEEmembership{Fellow,~IEEE}

\thanks{
Bing Luo  is with Shenzhen Institute of Artificial Intelligence and Robotics for Society, The Chinese University of Hong Kong, Shenzhen, China, and the Department of Electrical Engineering and Institute for Network Science, Yale University, USA. (e-mail: luobing@cuhk.edu.cn)}
\thanks{Xiang Li and Jianwei Huang (corresponding author) are with the School of Science and Engineering, The Chinese University of Hong Kong, Shenzhen,  China, and the Shenzhen Institute of Artificial Intelligence and Robotics for Society, Shenzhen, China. (e-mail:lixiang@cuhk.edu.cn; jianweihuang@cuhk.edu.cn)}
\thanks{Shiqiang Wang  is with IBM T. J. Watson Research Center, Yorktown Heights, NY, USA. (e-mail: shiqiang.wang@ieee.org)}
\thanks{Leandros Tassiulas  is with the Department of Electrical Engineering and Institute for Network Science, Yale University, USA. (e-mail: leandros.tassiulas@yale.edu)}
\thanks{The research of Bing Luo was supported by the AIRS-Yale Joint Postdoctoral Fellowship. The research of Xiang Li and Jianwei Huang was supported by the Shenzhen Science and Technology Program (JCYJ20210324120011032), Shenzhen Institute of Artificial Intelligence and Robotics for Society, and the Presidential Fund from the Chinese University of Hong Kong, Shenzhen. The research of
Leandros Tassiulas was supported by the NSF  CNS-2112562 AI Institute for Edge Computing Leveraging Next Generation Networks (Athena) and the ONR  N00014-19-1-2566. This paper was presented in part at the IEEE INFOCOM, Virtual Conference, 2021~\cite{luo2020cost}.}
}

\maketitle
\begin{abstract}
Federated learning (FL) is a distributed learning paradigm that enables a large number of mobile devices to collaboratively learn a model under the coordination of a central server 
without sharing their raw data. 
Despite its practical efficiency and effectiveness, 
{the 
iterative on-device learning process (e.g., local computations and global communications with the server)} 
incurs a considerable cost in terms of learning time and energy consumption, 
which depends crucially on the number of selected clients and the number of local iterations in each training round. 
In this paper, 
we analyze how to 
design adaptive FL {in mobile edge networks} that optimally chooses
these essential control variables  
to minimize the 
total cost  
while ensuring convergence. 
We establish the analytical
 relationship  between  the  
total   cost  and  the control variables with the convergence upper bound. To efficiently solve the cost minimization problem,  
we develop a low-cost sampling-based algorithm to learn the   convergence related unknown parameters. 
We derive important solution properties that effectively identify the 
design principles 
for different optimization metrics.  
Practically, we  evaluate our theoretical results  both in a simulated environment and on a hardware prototype. Experimental evidence verifies our derived properties and demonstrates that our proposed solution 
achieves near-optimal performance for different optimization metrics for  various datasets 
and heterogeneous system and statistical settings. 

\end{abstract}
\begin{IEEEkeywords}
Federated learning,  mobile edge networks, cost analysis, scheduling, optimization algorithm.
\end{IEEEkeywords}
\IEEEpeerreviewmaketitle

\section{Introduction}

With the rapid advancement of Internet of Things (IoT) and social networking applications, there is an exponential increase of data generated at network edge devices, such as smartphones, IoT devices, and sensors \cite{chiang2016fog}. These valuable data provide highly useful information for the prediction, classification,  and other intelligent applications, which can improve our daily lives \cite{shalev2014understanding}. To analyze and exploit the large amount of data, standard machine learning (ML) techniques normally require collecting the training data {in a central server}. 
However, such centralized data collection and training can be quite challenging to perform due to limited 
communication bandwidth  
and data privacy concerns \cite{mao2017survey,park2019wireless}. 

To tackle this challenge, 
\emph{Federated learning (FL)} has emerged as an attractive distributed learning paradigm, which
enables many clients\footnote{Depending on the type of clients, FL can be categorized into cross-device FL and cross-silo FL (clients are companies or organizations, etc.) \cite{kairouz2019advances}. This paper focuses on the former and we use ``device'' and ``client'' interchangeably.}
to collaboratively train a  model under the coordination of a central server, while keeping the training data decentralized and private\cite{mcmahan2017communication,kairouz2019advances}. 
In FL settings, 
the training data are in general massively distributed over a large number of clients, 
and the communications between the server and clients are typically operated at lower rates compared to datacenter settings. 
These unique features necessitate  FL algorithms to 
perform \emph{multiple local iterations} in parallel on \emph{a fraction of randomly sampled clients} and then aggregate the resulting model update via the central server periodically \cite{mcmahan2017communication}.\footnote{{Model compression is also an effective approach for improving FL communication efficiency. While orthogonal to the focus of this work, standard compression approaches such as sketched or random masking and model pruning \cite{jiang2020pruning,kairouz2019advances} can be used together to further reduce the cost.}}

{FL has demonstrated its effectiveness 
in various \emph{statistically heterogeneous} settings, 
e.g., unbalanced and {non-independent and identically distributed (non-i.i.d.)} data  \cite{bonawitz2019towards, 
li2019convergence,9099064,9165960}. 
Nevertheless, an efficient deployment of FL in mobile edge networks also needs to consider \emph{system heterogeneity}.
{This is because in mobile edge environment, the system bandwidth is limited and shared by all connected mobile devices with potential mutual interference. %
 Moreover, the selected devices may have different computational capabilities and dynamic wireless channel conditions due to mobility and channel fading, thus an efficient  scheduling strategy in FL should adapt to system heterogeneity.}}  

Because model training and information transmission for on-device FL can be both time and energy consuming,
it is necessary and important to analyze the \emph{cost} that is incurred for completing a given FL task. 
In  general,  the  cost  of  FL  includes multiple components such as  learning  time  
and  energy consumption \cite{tran2019federated}. 
The importance of different cost components depends on the characteristics of FL systems and applications. 
For example, in a solar-based sensor network, energy consumption is the major concern for the sensors to participate in FL tasks, whereas in a multi-agent search-and-rescue task where the goal is to collaboratively learn an unknown map, achieving timely result would be the first priority. 
Therefore, a cost-effective FL design needs to \emph{jointly optimize various cost components} (e.g., learning  time and  energy consumption) \emph{for different preferences}.

A way of optimizing the cost is to adapt control variables in the FL process to achieve a properly defined objective. For example, some existing works have considered the adaptation of communication interval (i.e., the number of local iterations between two 
global aggregation rounds) for communication-efficient FL with convergence guarantees \cite{wang2019adaptive,wang2018adaptive}. 
However, a limitation in these works is that they only adapt a single control variable (i.e., communication interval) in the FL process and ignore other essential aspects, such as the number of participating clients in each round, which can have a significant impact on the energy consumption.
\begin{figure}[!t]
	\centering
	\includegraphics[width=9cm,height=6.2cm]{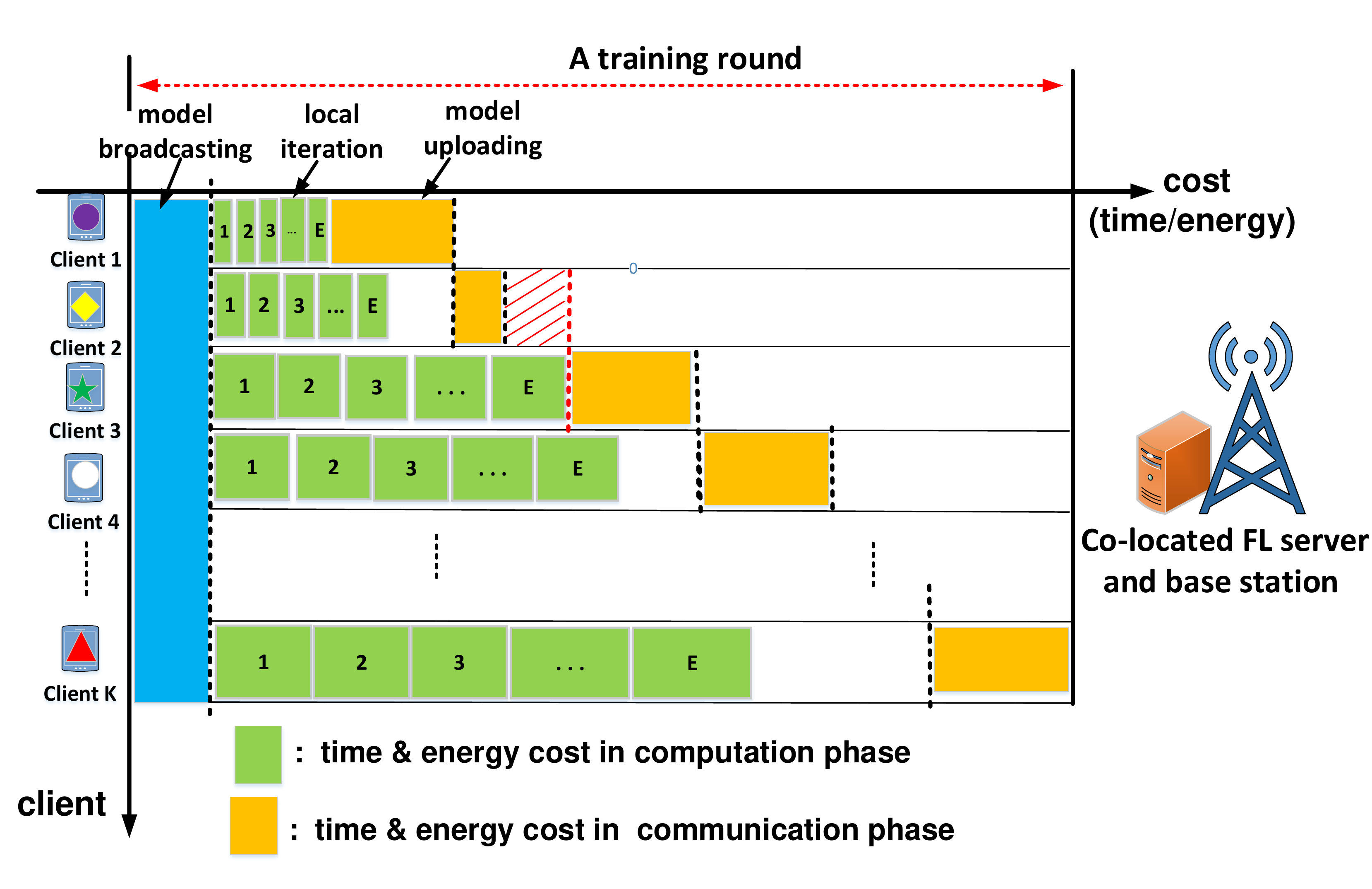}
 	\caption{A heterogeneous federated learning training round in  mobile network, where $K$ sampled clients performs $E$ steps of local iteration and communicates with the BS with time-sharing.}
	\label{fig:intro}
\end{figure}

In this paper,  we consider a \emph{multivariate} control problem for wireless FL with convergence guarantees. To minimize the expected cost, 
we develop an algorithm that adapts these control variables in the FL process to achieve our goal. 
Compared to the univariate setting in existing works, our \emph{multivariate} control problem is much more challenging due to the following reasons:  \emph{1)}~The choices of control variables are tightly coupled due to system and statistical heterogeneity. \emph{2)}~The relationship between the control variables and the learning convergence rate has only been captured by an upper bound with unknown coefficients in the literature. \emph{3)} Our cost objective includes multiple components (e.g., time and energy) which can have different importance depending on the 
application scenario, whereas existing works often consider a single optimization objective, e.g., minimizing the communication overhead.
In light of the above discussion, we state the key contributions of this work as follows:

\begin{itemize}
    \item  \emph{Cost-effective Wireless FL Design:}  
This work, for the first time, 
analyzes how to design adaptive FL  that optimally chooses the number of participating clients (${K}$) and the number of local iterations (${E}$)  to minimize the 
\emph{total} cost in mobile edge networks. 
Considering the wireless bandwidth limitation and interference, we propose a new scheduling scheme, as illustrated in Fig.~1, which characterizes both computation and communication heterogeneity. 

\item \emph{Optimization Algorithm:} We establish the analytical relationship between the 
    total cost, 
control variables, heterogeneous system parameters, and convergence upper bound, based on which we formulate and analyze the optimization problem for total cost minimization. 
When facing the challenging issue of estimating the  unknown coefficients in the convergence bound, 
we develop an effective sampling-based algorithm to learn these parameters with a small estimation overhead. 
    \item {\emph{Theoretical Properties
     :} We  obtain  important analytical properties  
     that effectively identify the design principles of $K$ and $E$ for  
     different 
     optimization metrics and heterogeneous system parameters. 
     Notably,  
     the choice of $K$ leads to an interesting trade-off between learning time reduction and energy saving, with a relatively large $K$ favoring the former while $K=1$ always benefiting the latter. 
{In contrast, the total cost, no matter emphasizing learning time or energy consumption, always first decreases and then increases in $E$.} 
We also show how heterogeneous system parameters in terms of computation and communication affect the optimal $K$ and $E$.}

    \item \emph{Simulation and Experimentation:} We evaluate our theoretical results 
    with real and synthetic heterogeneous datasets, both in a simulated wireless cellular network and on a WiFi-based hardware prototype. 
Experimental results demonstrate that our proposed optimization algorithm {provides near-optimal solution for different optimization metrics, 
and verify our theoretical findings for design principles and solution properties. 
}

\end{itemize}

The rest of the paper is organized as follows. We first present the related work in Section~\ref{sec:related}. Then, we present the system model and problem formulation in  Section~\ref{sec:systemModel}. In Section~\ref{sec:optimizationProblem}, we analyze the cost minimization problem and present an algorithm to solve it. We provide theoretical analysis on the solution properties in Section~\ref{sec:property}.
Experimentation results are given in Section~\ref{sec:experimentation} and the
conclusion is presented in Section~\ref{sec:conclusion}.

\begin{table*}[!tbp]
\label{relate_work}
\caption{{Summary of related work for FL cost minimization with control parameters}}
 \resizebox{\textwidth}{35mm}{\begin{tabular}{c||c|c||c|c|c|c}
\toprule[1.2pt]
\multicolumn{1}{c||}{\multirow{2}{*}{{reference}}} & \multicolumn{2}{c||}{{optimization goal}}                            & \multicolumn{4}{c}{{control parameters}}                                                                                                                     \\ \cline{2-7} 
\multicolumn{1}{c||}{} 
& {training time} & {energy consumption} 
& {local iteration number} & {sampled client number} & {mobile clients scheduling} & {resource allocation (CPU/power/bandwidth)} \\ \hline
[13]                                                 &          \checkmark            &              \checkmark                           &                                      &                            &      \checkmark     &      \checkmark                                   \\
\hline
[14]                                                 &          \checkmark            &              \checkmark                           &                             \checkmark           &                            &         &                                       \\
\hline
[15]                                                &   \checkmark                  &                                         &               \checkmark                       &                                     &   &                                   \\
\hline[27]                                                                   & \checkmark    &            &                         &                                  &  \checkmark  &     \checkmark                                           \\ \hline
[28]                                                                   & \checkmark &                                        &                                   & \checkmark  &                                    &                                           \\ \hline
[29]                                                &                  \checkmark  &     &                                                                     &   \checkmark   & \checkmark                                    &                                           \\ \hline
[30]                                                                    & \checkmark    &                                     &                                  &   & \checkmark                                    &                                 \checkmark           \\\hline
[31]                                                                    & \checkmark    &                                    &                                  &   & \checkmark                                    &            \checkmark                                \\\hline
[33]                                                                    & \checkmark    &                                    &                                  &   & \checkmark                                    &                                        \\\hline
[34]                                                                    &     &   \checkmark                                  &                                 &    & \checkmark                                    &       \checkmark                               \\\hline
[35]                                                                      &    &   \checkmark                                   &                                     &                                   &  \checkmark &            \checkmark                \\\hline
[36]                                                &                    & \checkmark    &                                                                          &                                   & \checkmark  &             \checkmark                     \\\hline
[37]                                                &                    & \checkmark    &                                                                          &                                  &  \checkmark   &             \checkmark              \\
\hline
[38]&                                  \checkmark                                   &    &                                                                     \checkmark       &                                &     &                              \\
\hline
[39]                                                &               \checkmark     &                                        &     \checkmark                                  &                                     &     &                         \\
\hline
[41]                                                &               \checkmark     &                                       &     \checkmark                                  &        \checkmark                               &     &                         \\
\hline
\textbf{this work}                                                &    \textbf{\checkmark}                &       \textbf{\checkmark}                                     &               \textbf{\checkmark}                       &    \textbf{\checkmark}                        &    \textbf{\checkmark}        &                                      \\
\midrule[1.2pt]
\end{tabular}}
\end{table*}

\section{Related Work}
\label{sec:related}
FL was first proposed in \cite{mcmahan2017communication}, which demonstrated FL's effectiveness of collaboratively learning a model without collecting users' data.
The de facto FL algorithm 
is federated averaging (FedAvg), which performs multiple local iterations in parallel 
 on  a subset of 
 devices in each round.
 A system-level FL framework \cite{bonawitz2019towards}  demonstrates the empirical success of  FedAvg in mobile devices. 
Recently, a convergence bound of FedAvg was established in \cite{li2019convergence}. 
{Other related distributed optimization algorithms are mostly analyzed for homogeneous and i.i.d. datasets (e.g.,  \cite{yu2018parallel,stich2018local,wang2018cooperative}) and full client participation (e.g., \cite{khaled2019first,smith2017federated}), which do not capture the essence of on-device FL.}  
{{Some extensions of FedAvg considered aspects such as adding a proximal term~\cite{li2018federated}, using gradient descent acceleration~\cite{liu2020accelerating},  
 variance reduction (e.g., \cite{karimireddy2020scaffold,liang2019variance}),  {or fine tuning hyper-parameters~\cite{8814558,reddi2020adaptive}}}. 
While novel insights are provided in these works, they did not consider optimization for cost/resource efficiency.}

{Literature in FL cost optimization mainly focused on learning time and on-device energy consumption in mobile edge networks.
The optimization of learning time was studied in \cite{   chen2020convergence,nguyen2020fast,nishio2019client,shi2020device, wadu2020federated,wang2020optimizing, yang2019scheduling,jiang2020pruning}, and
joint optimization for learning time and energy consumption was considered in \cite{mo2020energy,zeng2020energy,yang2019energy,luo2020hfel}.  These works considered resource (e.g., transmission power, communication bandwidth, and CPU frequency) allocation (e.g., \cite{ chen2020convergence,nguyen2020fast,mo2020energy,zeng2020energy,yang2019energy,luo2020hfel}), cost-aware client selection (e.g., \cite{nishio2019client, shi2020device}), client scheduling (e.g., \cite{wadu2020federated, wang2020optimizing,yang2019scheduling}), 
and model pruning\cite{jiang2020pruning} for \emph{pre-specified} (i.e., non-optimized) design parameters ($K$ and $E$ in our case) of the FL algorithm.}  

The optimization of a single design parameter $E$ 
 was studied in~\cite{wang2019adaptive,wang2018adaptive,han2020adaptive,luping2019cmfl,hsieh2017gaia}, most of which assume full client participation and can be infeasible for large-scale on-device FL. A very recent work in \cite{jin2020resource} considered the optimization of both $E$ and client selection for additive per-client costs. {However, the cost of learning time in our problem is non-additive on a per-client basis, since different clients perform local model updates in parallel. In addition, the convergence bound used in~\cite{jin2020resource} (also~\cite{tran2019federated}) is for a primal-dual optimization algorithm, which is different from the commonly used FedAvg algorithm. 
The challenge in optimizing both $K$ and $E$ for cost minimization of FedAvg, 
which also distinguishes our work from the above, is the need 
\emph{to analytically connect the total cost with the control variables as well as with the convergence rate}.} {The comparison of our work with the above related works is illustrated in Table~\ref{relate_work}, where we note that our optimization design is orthogonal to most works on resource allocation in the last column in Table~\ref{relate_work} and can be used together with those techniques to further reduce the cost.}

In addition, most existing work on FL are based on simulations, whereas we implement our algorithm in an actual hardware prototype with resource-constrained devices.

\section{
System Model and Problem Formulation}
\label{sec:systemModel}

We start by summarizing the basics of FL and its de facto algorithm FedAvg. Then, we introduce the scheduling strategy when applying FL in mobile edge networks with system heterogeneity. Finally, we present the cost optimization problem for FL tasks. We summarize all key notations in this paper in Table~II.

\subsection{Federated Learning}
Consider a scenario with a large number of mobile
clients that have data for training a machine learning  model. 
Due to data privacy 
concern, it is not desirable for clients to disclose and send their raw 
data to a high-performance data center. FL is a decentralized learning framework that aims to resolve this problem. 
Mathematically, FL is the following distributed optimization problem:
\begin{equation}
\label{gl_ob}
\min_{\mathbf{w}}  F\left( \mathbf{w} \right) :=\sum\nolimits_{k = 1}^N{p_k}{F_k}\left( \mathbf{w} \right)
\end{equation}
where the objective $F\left( \mathbf{w} \right)$ is also known as the global loss function, $\mathbf{w}$ is the model parameter vector, $N$ is the total number of devices, and $p_k>0$ is the weight of the $k$-th device such that $\sum\nolimits_{k = 1}^N p_k=1$. 
Suppose the $k$-th device has $n_k$ training data samples ($\mathbf{x}_{k, 1}, \cdots, \mathbf{x}_{k, n_{k}}$), and the total number of training data samples across $N$ devices is $n :=\sum\nolimits_{k \!=\! 1}^N n_k$, then we have $p_k=\frac{n_k}{n}$.
The local loss function of client $k$ is
\begin{equation}
\label{lo_ob}
{F_k}\left( \mathbf{w} \right) := \frac{1}{{{n_k}}}\sum\limits_{j =1}^{n_k} {{f}\left( \mathbf{w}; \mathbf{x}_{k,j} \right)},
\end{equation}
where $f(\cdot)$ represents a per-sample loss function, e.g., mean square error or cross entropy applied to the output of a model with parameter $\mathbf{w}$ and input data sample $\mathbf{x}_{k,j}$ \cite{wang2019adaptive}. 

FedAvg (Algorithm~1) was proposed in \cite{mcmahan2017communication} to solve \eqref{gl_ob}. In each \emph{round}~$r$, a subset of randomly selected clients $\mathcal{K}^{r}$ run $E$ steps\footnote{$E$ is originally defined as epochs of SGD in \cite{mcmahan2017communication}. In this paper we denote $E$ as the number of local iterations for theoretical analysis.} of stochastic gradient decent (SGD) on~\eqref{lo_ob} in parallel, {where $\mathcal{K}^{r} \subseteq \{1,2,...,N\}$}. Then, the updated model parameters of these $\left|\mathcal{K}^{r}\right|$ clients are sent to and aggregated by the server. This process repeats for many rounds until the global loss converges.

FL has demonstrated its effectiveness in tackling with statistical heterogeneity, e.g., unbalanced and non-i.i.d. data distribution. Meanwhile, applying it in practical mobile edge networks also needs to address clients' computational and communication heterogeneity. Therefore, an efficient FL design requires \emph{joint consideration of statistical and system heterogeneity}.

\begin{table}[!t]
 \label{keynotation}
	\centering
	\caption{Summary of all key notations}
	\begin{tabular}{l|l}
		\toprule[1pt] 
		$F\left( \mathbf{w} \right)$ & Global loss function \\
	    	${F_k}\left( \mathbf{w} \right)$ & Local loss function  \\
		$\mathbf{w}^*$ & Optimal model parameter that minimizes $F\left( \mathbf{w} \right)$\\
		$\mathbf{w}_R$ & Final model parameter after $R$ rounds\\
			$\epsilon$ & Desired precision with $\Expect[F(\mathbf{w}_R)]-F^{*} \le \epsilon$\\
			$R$ & Final round number for achieving $\epsilon$ \\
					$r$ & Round number index \\
		
		$\mathcal{K}^{r}$ & Randomly selected clients in round $r$  \\
		$K$ & Number of selected clients with $K\!:=\! \left| \mathcal{K}^{r}\!\right|\!$ \\
		$E$ & Number of local iteration steps\\
		$T^{r}$ &  Round time in round $r$ \\
		 $e^{r}$ &  Energy consumption in round $r$  \\
		$t_{k,p}$ &  Com\underline{p}utation  time of client $k$ for one iteration \\ 
			$e_{k,p}$ & Com\underline{p}utation  energy of client $k$ for one iteration\\
		$t_{k,m}^{r}$  & Co\underline{m}munication time of client $k$ in round $r$ \\
				$e_{k,m}^{r}$  & Co\underline{m}munication energy of client $k$ in round $r$\\
		  $\overline{t}_{k,m}$ & Expected	$t_{k,m}^{r}$  in $R$ rounds
		  \\
		   $\overline{e}_{k,m}$ & Expected  	$e_{k,m}^{r}$ in $R$ rounds
		   \\
			 $t_p$ & Expected per-device per-iteration computation time \\
			 $e_p$ & Expected per-device per-iteration computation energy\\
	  $t_m$ & Expected per-device per-round  communication time \\
		 $e_m$ & Expected per-device per-round  communication energy\\
		 $T_\textnormal{tot}$  & Total learning time after $R$ rounds \\
		 $e_\textnormal{tot}$  & Total energy consumption after $R$ rounds\\
	$T_k^{r}$  & Time after scheduling client $k$ in round $r$ \\
		$C_\textnormal{tot}$  & Balanced total cost  \\
		$\gamma$  & {Normalized price factor} \\
		$A_0$, $B_0$ & Unknown constants in 
		convergence bound\\
		$F_a$, $F_b$ & Pre-defined global loss for estimation $A_0$,  $B_0$  \\
		$R_a$, $R_b$ & Pre-defined round number for estimation $A_0$, $B_0$ \\
\bottomrule[1pt]
	\end{tabular}%
	\label{tabl}%
\end{table}%

\begin{algorithm}[t]
\small
	\caption{Federated  Learning Algorithm}
	\label{alg:fedavg}
	\KwIn{$K$, $E$, precision $\epsilon$, initial model $\mathbf{w_0}$, initial round index $r=0$}
	\KwOut{Final model parameter $\mathbf{w}_R$}
	\While{$r \le R$
	}{	Server randomly selects a subset of clients  $\mathcal{K}^{r}$ and broadcast the current global model parameter $\mathbf{w}_r$ to the selected clients\label{alg:fedavgStep1}\tcp*{Communication}
		
		Each selected client $k \in \mathcal{K}^{r} $ in parallel updates $\mathbf{w}_r$ by running $E$ steps  of SGD  on  \eqref{lo_ob} to compute a new model $\mathbf{w}_r^{(k)}$\label{alg:fedavgStep2}\tcp*{Computation}
		
		Each selected client $k \in \mathcal{K}^{r}$   sends back the updated model $\mathbf{w}_r^{(k)}$ to the server\label{alg:fedavgStep3}\tcp*{Communication}
		
		Server computes the new global model parameter $\mathbf{w}_{r+1} \leftarrow \frac{\sum_{k \in \mathcal{K}^{r}} p_k \mathbf{w}_r^{(k)}}{\sum_{k \in \mathcal{K}^{r}} p_k}  $\label{alg:fedavgStep4}\tcp*{Aggregation}
		
		$r \leftarrow r+1$;
	}
\end{algorithm}

\subsection{Deploying Federated Learning in Mobile Edge Networks}
In mobile edge networks, due to system bandwidth limitation and wireless interference, participating clients are usually scheduled in a time-sharing (TS) \cite{tran2019federated,mo2020energy} or frequency-sharing (FS) \cite{shi2020device, chen2020convergence,yang2019energy} protocols.\footnote{{We consider cellular network based one-hop communication between the server and the clients, since its topology is well-aligned with that in mainstream FL community. For FL in a decentralized topology, e.g., in wireless ad-hoc networks, the communications could be multi-hop, which is beyond our focus in this work.} Although leveraging analog aggregation techniques in wireless communications \cite{zhu2019broadband, yang2020federated, amiri2020machine} can increase the communication efficiency, it may require stringent synchronization and additional superposition code design.} 
{However, 
the selected heterogeneous clients may have diverse
  communication  and  computation capabilities, which results in different per-round training time under different scheduling strategies, especially for clients with dynamic wireless channels.}  This is because the per-round time depends on the slowest client (known as straggler), as 
  the server needs to collect all updates from the sampled clients before performing global aggregation. 

{In this paper, as shown in Fig. 1, we propose a new TS based FL model which works efficiently with  heterogeneous computation and communication time. We note that the empirical results in Section~\ref{sec:experimentation} demonstrate that our proposed TS is superior to existing FS schemes in \cite{zhu2019broadband, yang2020federated, amiri2020machine}. {The main reason is that, in those FS schemes, the frequency allocation for the sampled clients in each round is \emph{static}, which causes a waste in bandwidth resources. This is because, due to computational heterogeneity, 
clients who complete computation early can be allocated with more bandwidth for transmission before the next client completing computation, instead of always sharing the entire bandwidth static with other slower clients.}} In the following, we first describe our system model and then present the proposed TS protocol.  

\subsubsection{\textbf{System Model}} 
Similar to existing works \cite{mcmahan2017communication, li2019convergence, li2018federated}, 
we sample $K$ clients in each round $r$ (i.e., $K := \left| \mathcal{K}^{r} \right|,{ 1\le K \le N}$), where the sampling is uniform at random (without replacement) out of all $N$ clients. 
Similar to existing studies on wireless FL (e.g., \cite{ shi2020device, chen2020convergence,yang2019energy}), we assume that a particular device has the same  com{p}utation cost (time and energy) over multiple rounds, whereas the communication  cost 
varies between rounds due to wireless dynamics. 
We do not consider the server's downlink cost for model broadcasting and model aggregation, as we mainly focus on the performance bottleneck of the  battery-constrained edge devices.

\subsubsection{\textbf{Proposed Time-sharing Protocol}}
Consider a general round $r$, where the server (co-located with the base station) first broadcasts the current global model to the randomly selected $\mathcal{K}^{r}$ clients. 
Then, each of the selected clients performs $E$ steps of local iterations \emph{in parallel}, where we denote 
$t_{k,p}$ as the com\underline{p}utation time for client $k$ to perform one local iteration. 
In the communication phase, however, the selected clients \emph{sequentially} upload their model to the server in different time slots, where we denote $t_{k,m}^{r}$ as the co\underline{m}munication time for client $k$ to upload the model parameter in round $r$.\footnote{We note that $t_{k,m}^{r}$ can be calculated by $t_{k,m}^{r}=\left.{\Omega} \middle/ {B\log \left( {1 + \frac{p_k\bar{h}_k^{r}}{N_0}}\right)}\right.$, where  $\Omega$ is the size of model parameter, $B$ is the bandwidth, $p_k$ is the transmission power (we do not consider power control in this work, though it can be incorporated in future work), 
$\bar{h}_k^{r}$ is the average channel gain
of the client $k$ during the FL training time in round $r$ and $N_0$ is the  white noise power.}
 
Due to the difference between $t_{k,p}$ and $t_{k,m}^{r}$ among the selected $K$ clients in $\mathcal{K}^{r}$, the  time of round $r$  depends on the \emph{scheduling order} of the $K$ clients.
We present the optimal TS scheduling with the following theorem, which achieves the minimum time for a particular round $r$  with the selected clients $\mathcal{K}^{r}$. {We note that the proposed scheduling scheme mainly considers the total learning time for the sampled clients rather than their energy consumption. This is because the total energy consumption is the sum of all sampled clients' energy  cost,  which is independent of the scheduling order of the $K$ clients.}   

\newtheorem{theorem}{\textbf{Theorem}}
\begin{theorem}
\label{theorem1}
\textbf{(Optimal time-sharing scheduling)}
For any sampled clients set $\mathcal{K}^{r}$ in round $r$,
without loss of generality, we assume that  $\mathcal{K}^{r}$ is ordered based on $\{t_{k,p}: \forall k \in \mathcal{K}^{r}\}$, such that $t_{1,p}\le  \ldots \le t_{k,p}\le \ldots \le t_{K,p}$. Then, sequentially scheduling the ordered $K$ clients yields the minimum time of round $r$, compared to any other scheduling sequence. Under this scheduling sequence, denoting $T_k^{r}$ as the time after scheduling client $k$, we have 
\begin{equation}
\label{tdd}
       T_k^{r}=\max \left\{ {{t_{k,p}E}, {T_{k-1}^{r}}} \right\} + {t_{k,m}^{r}},\   k\in \left\{ {1, \ldots K} \right\},
\end{equation}
where $T_K^{r}$ is the entire time of round $r$ after scheduling all $K$ clients, and we define $T_0^{r}:=0$ and $T^{r} := T_K^{r}$ for convenience. 
\end{theorem}
We give the proof in Appendix \ref{theorem1app}. 
Compared to existing TS scheduling strategy in \cite{tran2019federated} where communication from the clients to the server will not begin until  all sampled clients complete local iterations, the scheduling policy specified  in  Theorem~\ref{theorem1} enables each client to upload its result to the server as soon as it finishes the computation when the wireless channel is available,  
which potentially reduces system waiting time and thus the round time $T^{r}$. 
The \emph{max} function in \eqref{tdd} characterizes whether the channel is vacant when a certain client completes its local iteration, e.g., the red shadow time slots in Fig.~1.

\subsection{Cost Analysis of Federated Learning in Mobile Edge Networks}
Despite the effectiveness and efficiency of FL,  practitioners need to take into account the cost of completing FL tasks. The total cost of FL, according to Fig.~1, involves  \emph{learning time}  and  \emph{energy consumption}, both of which are consumed during local computation  and global communication in each round.  

\subsubsection{\textbf{Time Cost}}
Based on the optimal TS scheduling in Theorem \ref{theorem1}, the per-round time after scheduling the selected $K$ 
 clients is $T^{r}$, 
which depends on the number of sampled clients $K$ and the number of location computation steps $E$.  
The total learning time $T_\textnormal{tot}$ after $R$ rounds 
is
\begin{equation}
\label{Ttot}
T_\textnormal{tot}(K,E,R)=\sum\nolimits_{r=1}^{R}T^{r}.
\end{equation}

\subsubsection{\textbf{Energy Cost}}
Similar to the notation for time cost, 
by denoting $e_k$ as the per-round energy consumption  for client $k$ to complete the computation and communication, we have
\begin{equation}
\label{energyk}
e_{k}^{r}=e_{k,p}  E+ e_{k,m}^{r},
\end{equation}
where $e_{k,p}$ is the com\underline{p}utation energy for client $k$ to perform one step local iteration in each round, and $e_{k,m}^{r}$ is client $k$'s  com\underline{m}unication energy in round $r$ due to wireless dynamics. 

Unlike the 
per-round time cost, the energy cost $e^{r}$ in each round $r$ depends on the \emph{sum} energy consumption of the selected $K$ clients. 
Therefore, the total energy cost $e_\textnormal{tot}$ after $R$ rounds 
can be expressed as
\begin{equation}
\label{etot}
e_\textnormal{tot}(K,E,R)=\sum\nolimits_{r=1}^{R}\sum\nolimits_{k \in \mathcal{K}^{r}}e_{k}^{r}.
\end{equation}

\subsection{Problem Formulation}
Considering the difference of the two cost metrics, the optimal solutions of $K$, $E$ and $R$ generally do not achieve the common goal for minimizing both $T_\textnormal{tot}$ and $e_\textnormal{tot}$. 
To strike the balance between learning time and energy consumption, we introduce a weight $\gamma \in \left[0, 1\right]$ and optimize the balanced  cost function in the following form: 
\begin{equation}
\label{homocost}
    C_\textnormal{tot}(K,E,R)= \gamma  e_\textnormal{tot}(K,E,R)+\left(1-\gamma\right)
    T_\textnormal{tot}(K,E,R), 
    \end{equation}
{where $1-\gamma$  and $\gamma$ can be interpreted as the \emph{normalized price} of the two costs, i.e., how much monetary cost for one unit of time and one unit of energy, respectively. The value of $\gamma$ can be adjusted for different preferences.} 
For example, we can set $\gamma=0$ when all clients are plugged in and energy consumption is not a major concern, whereas $\gamma=1$ when devices are solar-based sensors where saving the devices' energy is the priority.   

Our goal is to minimize the {expected} total cost while ensuring convergence, 
which translates into this problem:
\begin{equation}\begin{array}{cl}
\label{ob1}
\!\!\!\!\!\!\!\!\textbf{P1:}\quad\quad \min_{E, K, R} & \Expect[C_\textnormal{tot}(E,K,R)] \\
\quad\quad \text { s.t. } & \Expect[F(\mathbf{w}_R)]-F^{*} \le \epsilon,\\
& K,E,R \in \mathbb{Z}^{+}, \  \text{and}\ \  1 \le K \le N. 
\end{array}\end{equation}
{where $\Expect[F(\mathbf{w}_R)]$ is the expected loss after $R$ rounds, 
$F^*$ is the (true and unknown) minimum value of $F$, and $\epsilon$ is the desired precision. {We note that the expectation in Problem \textbf{P1} takes over three source of randomness, where the first two  randomness come from the client sampling and data sampling in SGD in each round, and the third randomness comes  from the varying communication cost in each round due to dynamic wireless channel conditions.}}

{Solving Problem  \textbf{P1} is challenging in two aspects. First, it is difficult to find an \emph{exact analytical expression} to relate  $K$, $E$ and $R$ with $C_\textnormal{tot}$, especially due to the \emph{
maximum} function in $T_\textnormal{tot}$. 
Second, it is generally impossible to obtain an exact analytical relationship
to connect $K$, $E$ and $R$  with the convergence constraint. 
In the following section, we propose an algorithm that approximately solves Problem  \textbf{P1}, which we later show with extensive experiments that the proposed solution can achieve a  near-optimal performance of Problem  \textbf{P1}.} 

\section{Cost-Effective Optimization Algorithm}
\label{sec:optimizationProblem}

This section shows how to approximately solve Problem  \textbf{P1}. 
{We first formulate an 
alternative problem that  
includes an approximate 
analytical relationship between the expected cost $\Expect[C_\textnormal{tot}]$, the convergence constraint, and the control variables $E$, $K$ and $R$. 
Then, we show that this new optimization problem 
can be efficiently solved after estimating 
unknown parameters associated with the convergence bound, and we propose a sampling-based algorithm to learn these unknown parameters.}
\subsection{Approximate Solution to Problem  \textbf{P1}}
\subsubsection{\textbf{Analytical Expression of}  $\Expect[e_\textnormal{tot}]$} We first 
analytically establish the expected energy cost $\Expect[e_\textnormal{tot}]$ with $K$ and $E$.
\newtheorem{lemma}{{Lemma}}
\begin{lemma}
\label{lemma:expect_e_tot}
The expectation of $e_\textnormal{tot}$ in \eqref{etot} can be expressed as
\begin{equation}
\label{het_energy}
   \Expect[e_\textnormal{tot}(K,E,R)]=  K\left(e_{p}E+ e_{m}\right)R,
\end{equation}
where $e_p\!:=\!\frac{\sum_{k=1}^{N} e_{k,p}}{N}$ and $e_m\!:=\!\frac{\sum_{k=1}^{N}\sum_{r=1}^{R}{e}_{k,m}^{r}}{NR}$ denote the average per-device energy consumption for one local iteration and one round of communication, respectively.
\end{lemma}
\begin{proof}
Given that all devices are sampled uniformly at random in each round and each round $K$ out of $N$ clients are sampled, thus for $R$ rounds, each device will be sampled in $\frac{KR}{N}$ rounds in expectation. Based on the result in \eqref{energyk}  that each device $k$ consumes $e_{k,p}  E+ {e}_{k,m}^{r}$ energy in each round, 
summing up $e_{k,p}  E+ {e}_{k,m}^{r}$ over all $N$ clients and over $\frac{KR}{N}$ rounds lead to \eqref{het_energy}.
\end{proof}

\subsubsection{\textbf{Analytical Expression of} $\Expect[T_\textnormal{tot}]$} 

{Next, we give an approximate expression of 
the expected time cost $\Expect[T_\textnormal{tot}]$ with $K$ and $E$. 
{In the following, we first show how to approximate the  per-round  time  $T^r$ in \eqref{tdd}, based on which we formulate a combinatorial expression of expected $T_\textnormal{tot}$ in \eqref{Ttot}. Then, we further approximate the expectation of $T_\textnormal{tot}$ for analytical tractability.}} 

{Considering that the communication time is usually the bottleneck in wireless FL settings, 
it is most likely that the maximum in \eqref{tdd}  is equal to $T_{k-1}^{r}$ for any $k  \in \{2, \ldots, K\} $, e.g., the red shadow time slot in Fig.~\ref{fig:intro} rarely happens. In other words, the communication channel will be usually occupied after the fastest client (e.g., $k=1$ in the $K$ clients reordered as in Theorem~\ref{theorem1}) finishes computation. 
Therefore, 
$T^{r}$ in \eqref{tdd} can be approximated by 
\begin{equation}
\label{app_tr_wireless}
       T^{r}\approx
      {t_{1,p}}E+\sum_{k=1}^K{t_{k,m}^{r}}.
\end{equation}}
Note that the ordering of $\{t_{k,p}: \forall k \in \mathcal{K}^{r}\}$ varies across different rounds due to random sampling, i.e., the client with index $k=1$ may be different in different rounds. With a slight abuse of notation, we use $t_{i,p}$ to denote the $i$-th fastest client (in terms of computation) out of all $N$ clients (i.e., before sampling), 
such that
\begin{equation}
\label{reorder}t_{1,p} \leq t_{2,p} \leq \ldots \leq t_{i,p}  \leq \ldots \leq t_{N,p}.
\end{equation}

\begin{lemma}
\label{lemma:expect_t_tot}
With the reordered clients 
as in \eqref{reorder} and the approximate  $T^{r}$ as in \eqref{app_tr_wireless}, the expectation of $T_\textnormal{tot}$ in \eqref{Ttot} can be expressed as\footnote{The notation of $C_{N}^{K}$ is also noted as $\binom{N}{K}$ which represents the combination number of choosing $K$ out of $N$ without replacement.} 
\begin{equation}
\label{avgtime}
\Expect[T_\textnormal{tot}(K,E,R)]\approx\left(\frac{\sum\nolimits_{i=1}^{N-i+1}{C_{N-i}^{i-1}}t_{i,p} }{C_{N}^{K}}E
+t_mK\right)R,
\end{equation}
where $t_m:=\frac{\sum_{i=1}^{N}\sum_{r=1}^{R}{t}_{i,m}^r}{NR}$ is the average per-device time cost for one round of communication. 
\end{lemma}
We give the full proof in Appendix \ref{lemma2app}. 
The basic idea is to show the expectation of $T^{r}$ in \eqref{app_tr_wireless} is 
$\Expect[T^{r}] \approx \frac{\sum\nolimits_{i=1}^{N-K+1}{C_{N-i}^{K-1}}t_{i,p} }{C_{N}^{K}}E+t_mK,$
where $\frac{\sum\nolimits_{i=1}^{N-K+1}{C_{N-i}^{K-1}}t_{i,p} }{C_{N}^{K}}E$ 
represents that the probability of client $i$ being the first 
scheduling client is
$\frac{{C_{N-i}^{K-1}}}{C_{N}^{K}}$, and the second term $t_mK$ is derived similar to $e_m$ in \emph{Lemma} 1. In particular, the multiplication of $K$ in $t_mK$ characterizes the scheduling property of wireless communications with limited bandwidth.

However,  $\Expect[T_\textnormal{tot}]$ in \eqref{avgtime} is still hard to analyze due to the various combinatorial terms with respect to control variable $K$. For analytical tractability, similar to how we derive \eqref{het_energy}, we define an approximation of $\Expect[T_\textnormal{tot}]$ as
\begin{equation}
    \label{appro_t_tot}
    \tilde{\Expect}[T_\textnormal{tot}(K,E,R)] :=\left(t_pE+t_mK\right)R,
\end{equation}
where $t_p:=\frac{\sum_{i=1}^{N} t_{i,p}}{N}$ is the average per-device time cost for one local iteration. 
The approximation $\tilde{\Expect}[T_\textnormal{tot}]$ in \eqref{appro_t_tot} is equivalent to $\Expect[T_\textnormal{tot}]$ in \eqref{avgtime} in the following two cases.

\noindent\emph{Case 1}: For clients with homogeneous computation capabilities, e.g., $t_{p}=t_{i,p}$, $\forall i \in \{1, \ldots, N\}$, along with the recursive property of  $C_{m}^{n}\!+C_{m}^{n-1}\!=\!C_{m+1}^{n}$, 
we have 
\begin{equation}
\notag
\label{avgtimehomo}
\begin{array}{cl}
\Expect[{T_\textnormal{tot}}(K,E,R)]&=\left(t_pE\frac{\sum\nolimits_{i=1}^{N-K+1}{C_{N-i}^{K-1}}}{C_{N}^{K}}+t_mK\right)
R\\
&=\left(t_p E+t_mK\right)R\\ &=\tilde{\Expect}[T_\textnormal{tot}(K,E,R)].
\end{array}
\end{equation}
\noindent\emph{Case 2}: For clients with heterogeneous computation capabilities with
 $K=1$, we have 
\begin{equation}
\notag
\label{avgtimek=1}
\begin{array}{cl}
\Expect[{T_\textnormal{tot}}(K=1,E,R)]&=\left(\frac{\sum\nolimits_{i=1}^{N}{t_{i,p}E}}{N}+t_mK\right)R\\
&=\left(t_p E+t_mK\right) R\\
&=\tilde{\Expect}[T_\textnormal{tot}(K=1,E,R)].
\end{array}
\end{equation}

{In the following cost analysis, we use $\tilde{\Expect}[T_\textnormal{tot}(K,E,R)]$ in \eqref{appro_t_tot} as the approximate expected value of the original total learning time in \eqref{Ttot}.}

\subsubsection{\textbf{Analytical
Relationship Between $\Expect[C_\textnormal{tot}]$ and Convergence}}
Based on 
$\Expect[e_\textnormal{tot}]$ in \eqref{het_energy} and the approximation $\tilde{\Expect}[T_\textnormal{tot}(K,E,R)]$ in \eqref{appro_t_tot}, 
we formulate an approximate objective function of Problem  \textbf{P1} as
\begin{equation}
\label{obj_fun} 
\tilde{\Expect}[C_\textnormal{tot}]=  \left(\gamma K\left( e_pE+ e_m\right)+\left(1-\gamma\right)\left(t_pE+t_mK\right)\right)R.
\end{equation}
{To connect 
$\tilde{\Expect}[C_\textnormal{tot}]$ with the $\epsilon$-convergence constraint in \eqref{ob1} as well as to characterize the heterogeneous data, we utilize the 
convergence result \cite{li2019convergence}}\footnote{{We use the convergence result in \cite{li2019convergence} because both \cite{li2019convergence} and our work analyze the FedAvg algorithm, and the convergence result in \cite{li2019convergence} includes both  control variables $E$ and $K$ for cost minimization. We also note that the convergence result in \cite{li2019convergence} requires clients to decay the learning rate throughout the training process, which we show later in Section VI.}}: 
\begin{equation}
    \label{upperbound1}
    \Expect[F(\mathbf{w}_R)]\!-\!F^{*}\le \frac{1}{ER}\left(A_0+B_0\left(1+
 \frac{N-K}{K(N\!-\!1)}\right) E^2\right),
\end{equation} 
where $A_0$ and 
$B_0$ 
are loss function related constants characterizing the statistical  heterogeneity of non-i.i.d. data.
By letting the upper bound satisfy the convergence constraint,\footnote{We note that optimization using upper bound as an approximation has also been adopted in \cite{wang2019adaptive} and resource allocation based literature \cite{tran2019federated, chen2020convergence,yang2019energy}.} we approximate Problem \textbf{P1} as
\begin{equation}\begin{array}{cl}
\label{ob2}
\!\!\!\!\!\!\textbf{P2:} \  \min_{E, K, R} & \!\!\!  \left(\gamma K\left( e_pE+ e_m\right)+\left(1\!-\!\gamma\right)\left(t_pE+t_mK\right)\right)R \\
\quad \text { s.t. } &\!\!\! \frac{1}{ER}\left(A_0+B_0\left(1+
 \frac{N-K}{K(N-1)}\right) E^2\right) \le \epsilon\\
& K,E,R \in \mathbb{Z}^{+}, \  \text{and}\ \  1 \le K \le N.
\end{array}\end{equation}
Combining with 
\eqref{upperbound1}, we can see that Problem  \textbf{P2} is more constrained than Problem \textbf{P1}, as any feasible solution  of  \textbf{P2} is also feasible for \textbf{P1}, but not vice versa.

Problem \textbf{P2}, however, is a non-linear constrained \emph{integer} optimization problem, which is still difficult to solve in general.   Therefore, we relax $K$, $E$ and $R$ as continuous variables for theoretical analysis, which are rounded back to integer variables\footnote{{For two integer variables we have four rounding combinations of $\left(\left \lceil{K}\right \rceil, \left \lceil{E}\right \rceil \right)$, $\left(\left \lceil{K}\right \rceil, \left \lfloor{E}\right \rfloor \right)$, $\left(\left \lfloor{K}\right \rfloor, \left \lceil{E}\right \rceil \right)$, and $\left(\left \lfloor{K}\right \rfloor, \left \lfloor{E}\right \rfloor \right)$.}} later. For the relaxed problem, if any feasible solution $E^\prime, K^\prime$, and $R^\prime$ satisfies the $\epsilon$-constraint in Problem \textbf{P2} with inequality, we can always decrease this $R^\prime$ to some $R^{\prime\prime}$ ($R^{\prime\prime} < R^\prime$), which satisfies the constraint with equality but reduces the objective function value. Hence, for optimal $R$, the $\epsilon$-constraint is always satisfied with equality, and we can obtain $R$ from this equality as 
\begin{equation}
R = \frac{1}{\epsilon E}\left(A_0+B_0\left(1+
 \frac{N-K}{K(N-1)}\right) E^2\right).
 \label{eq:R_solution}
\end{equation}
By 
substituting \eqref{eq:R_solution} into problem  \textbf{P2}, we obtain 
\begin{equation}
\label{ob3}
\begin{array}{cl}
\!\!\!\textbf{P3:} \ 
\min_{E, K}  & \!\!\!\!\!\left(\left(1\!-\!\gamma\right)\left(t_pE+\!t_mK\right)\!+\!\gamma K \left(e_p E\!+\!e_m\right)\right) \\&\cdot \frac{\left(\!A_0+B_0\left(\! 1+
 \frac{N-K}{K(N-1)}\!\right) E^2\right)}{\epsilon E}  \\
\!\!\!\!\!\quad\text {s.t.} &  \!\!\! {E}\ge 1, \  \text{and}\ \  1 \le K \le N, 
\end{array}
\end{equation}

In the following, we solve Problem \textbf{P3} as an approximation of the original problem \textbf{P1}. 
Our empirical results in Section~\ref{sec:experimentation} demonstrate that the solution obtained from solving  \textbf{P3} 
achieves \emph{near-optimal performance} of the original problem \textbf{P1}. 
For ease of analysis, we incorporate $\epsilon$ in the constants $A_0$ and $B_0$.

\subsection{Solving the Approximate Optimization Problem  \textbf{P3}}
In this subsection, we first characterize some properties of the optimization problem \textbf{P3}. Then, we propose a sampling-based algorithm to learn the problem-related unknown parameters $A_0$ and $B_0$, based on which the  solution $K^*$ and $E^*$ (of Problem \textbf{P3}) can be efficiently computed. The overall algorithm for obtaining $K^*$ and $E^*$ is given in Algorithm~\ref{alg:optimalSolution}.

\begin{algorithm}[t]
\small	\caption{Cost-effective design of $K$ and $E$}
\label{alg:optimalSolution}
	\KwIn{{$N$, $\gamma$, $t_p$, $t_m$, $e_p$, $e_m$, {loss $F_a$ and $F_b$,  $\mathbf{w}_0$, number of sampled pairs $M$}}, stopping condition $\epsilon_0$}
	\KwOut{{$K^*$ and $E^*$}}

\For{$i=1,2, \ldots, M$ \label{alg:optimalSolution:startEstimation}}{
Empirically choose $(K_i$, $E_i)$ and run Algorithm~1;

Record $R_{i,a}$ and $R_{i,b}$ when $F_a$ and $F_b$ are reached;
}

Calculate average $\frac{A_0}{B_0}$ using \eqref{A0B05}; \label{alg:optimalSolution:endEstimation}

Choose a feasible 
$z_0 \leftarrow \left(K_0, E_0\right)$ and set $j \leftarrow 0$; \label{alg:optimalSolution:startOptimization}

\While{$\Vert z_{j} - z_{j-1}\Vert > \epsilon_0$}{ 
Substitute $E_{j}$, $\frac{A_0}{B_0}$, $N$, $\gamma$, $t_p$, $t_m$, $e_p$, $e_m$ into  \eqref{ob3} 
and derive $K^\prime$;
 
 
$K_{j+1} \leftarrow \arg\min_{K \in [1, N]} |K - K'|$; \label{alg:optimalSolution:projectionK}

Substitute $K_{j+1}$,  $\frac{A_0}{B_0}$, $N$, $\gamma$, $t_p$, $t_m$, $e_p$, $e_m$ into \eqref{ob3} 
and derive $E^\prime$;


$E_{j+1} \leftarrow \arg\min_{E \geq 1} |E - E'|$; \label{alg:optimalSolution:projectionE}

$z_{j+1} \leftarrow \left(K_{j+1}, E_{j+1}\right)$ and $j \leftarrow j+1$;
}

Substitute four rounding combinations of $\left(\left \lceil{K_j}\right \rceil, \left \lceil{E_j}\right \rceil \right)$, $\left(\left \lceil{K_j}\right \rceil, \left \lfloor{E_j}\right \rfloor \right)$, $\left(\left \lfloor{K_j}\right \rfloor, \left \lceil{E_j}\right \rceil \right)$, and $\left(\left \lfloor{K_j}\right \rfloor, \left \lfloor{E_j}\right \rfloor \right)$ into the objective function of problem  \textbf{P3}, and set the pair with the minimum value as $\left(K^*, E^*\right)$ \label{alg:optimalSolution:rounding}

\Return $\left(K^*, E^*\right)$
\label{alg:optimalSolution:endOptimization}



\end{algorithm}

\subsubsection{\textbf{Characterizing Problem  \textbf{P3}}}

\begin{theorem}
\label{theorem:biconvex}
Problem \textbf{P3} 
is strictly biconvex \cite{gorski2007biconvex}. 
\end{theorem}
\begin{proof}
For any $E \ge 1$, we have
{\small
 \begin{equation*}
\dfrac{\partial^2 \tilde{\Expect}[C_\textnormal{tot}]}{\partial^2 K}=\dfrac{2(1-\gamma)B_{0} Nt_{p} E^{2}}{(N-1) K^{3}}>0.\end{equation*} }
Similarly, for any $1\le K \le N$, we have
{\small
\begin{equation}
\begin{array}{cl}
\dfrac{\partial^2 \tilde{\Expect}[C_\textnormal{tot}]}{\partial^2 E}=2 \left( \left(1-\gamma\right)t_p+\gamma K e_p\right) B_{0}\left(1+\frac{N-K}{K(N-1)}\right)
\\\qquad +
\dfrac{2A_{0}\left( (1-\gamma)Kt_m+\gamma Ke_m\right) }{E^{3}}>0
\end{array}
\end{equation}}%
Since the domain of $K$ and $E$ is convex,  we conclude that Problem \textbf{P3} is strictly biconvex.
\end{proof}

The  biconvex property allows many efficient algorithms, such as \emph{Alternate Convex Search} (ACS) approach, to a achieve a guaranteed local optima\cite{gorski2007biconvex}. 
E.g., we could iteratively solve $K$ and $E$ in the
objective function  \eqref{ob3} until we achieve the converged $K^\ast$ and $E^\ast$. This optimization process corresponds to Lines~\ref{alg:optimalSolution:startOptimization}--\ref{alg:optimalSolution:endOptimization} of Algorithm~\ref{alg:optimalSolution}, where Lines~\ref{alg:optimalSolution:projectionK} and \ref{alg:optimalSolution:projectionE} ensure that the solution is taken within the feasibility region, and Line~\ref{alg:optimalSolution:rounding} rounds the continuous values of $K$ and $E$ to integer values.

\subsubsection{\textbf{Estimation of Parameters $\frac{A_0}{B_0}$}}
Solving the objective function \eqref{ob3} is still nontrivial because it    
includes unknown parameters $A_0$ and $B_0$, which can only be determined during the learning process.\footnote{We assume that 
$t_p$, $t_m$, $t_m$ and $e_m$ can be measured offline.} In fact, the optimal $K$ and $E$ in \eqref{ob3} only depend on the value of $\frac{A_0}{B_0}$ as we could divide $B_0$ and incorporate it in $\epsilon$. 
In the following, we propose a sampling-based algorithm to estimate $\frac{A_0}{B_0}$, and show that the overhead for estimation is marginal.

The basic idea is to sample different combinations of $\left(K, E\right)$ and use the upper bound in \eqref{upperbound1} to approximate $F(\mathbf{w}_R)\!-\!F^{*}$. 
{Specifically, we empirically sample\footnote{{Our sampling criteria is to cover diverse combinations of  $\left(K,E\right)$.}}} 
a pair $\left(K_i, E_i\right)$ and run Algorithm~1 with an initial model $\mathbf{w}_0\!=\!\mathbf{0}$ until it reaches two pre-defined global losses $F_a := F(\mathbf{w}_{R_{i,a}})$ and $F_b := F(\mathbf{w}_{R_{i,b}})$ ($F_b<F_a$), where $R_{i,a}$ and ${R_{i,b}}$ are the executed round numbers for reaching losses $F_a$ and $F_b$. The pre-defined losses $F_a$ and $F_b$ can be set to a relatively high value, to keep a small estimation overhead, but they cannot be too high either as it would cause low estimation accuracy. 
Then, we have
\begin{equation}
\begin{cases}
    \label{A0B01}
    R_{i,a} \approx d + \dfrac{A_0 + B_0\left(1+
 \frac{N-K_i}{K_i(N-1)}\right) E_i^2}{E_i\left(F_a-F^{*}\right)},\\
    {R_{i,b}} \approx d + \dfrac{A_0 + B_0\left(1+
 \frac{N-K_i}{K_i(N-1)}\right) E_i^2}{E_i\left(F_b-F^{*}\right)}.
 \end{cases}
\end{equation}
from \eqref{upperbound1}, where $d$ captures a constant error of using the upper bound to approximate $F(\mathbf{w}_R)\!-\!F^{*}$. 
Based on \eqref{A0B01}, we have
 \begin{equation}
    \label{A0B03}
    {R_{i,b}}-R_{i,a} \approx 
  \frac{\Delta}{E_i}  \left(    {A_0\! +\! B_0\left(1\!+\!
 \frac{N-K_i}{K_i(N\!-\!1)}\right)\! E_i^2}\right),
\end{equation}
where $\Delta :=\frac{1}{F_b\!-\!F^{*}}\!-\!\frac{1}{F_a\!-\!F^{*}}$. Similarly, sampling another pair of ($K_j$, $E_j$) and performing the above process gives us another executed round numbers $R_{j,a}$ and $R_{j,b}$. 
Thus, we have
 \begin{equation}
    \label{A0B05}
    \frac{E_i\left(R_{i,b} -R_{i,a}\right)}{E_j\left(R_{j,b}-R_{j,a}\right)} \approx 
    \frac{ {A_0 + B_0\left(1+
 \frac{N-K_i}{K_i(N-1)}\right) E_i^2}}{{A_0 + B_0\left(1+
 \frac{N-K_j}{K_j(N-1)}\right) E_j^2}}.
\end{equation}
We can obtain $\frac{A_0}{B_0}$ from \eqref{A0B05} (note that the variables except for $\frac{A_0}{B_0}$ are known).
In practice, 
we may sample several different pairs of $\left(K_i, E_i\right)$ 
to obtain an averaged estimation of $\frac{A_0}{B_0}$. This estimation process is given in Lines~\ref{alg:optimalSolution:startEstimation}--\ref{alg:optimalSolution:endEstimation} of Algorithm~\ref{alg:optimalSolution}.


\textbf{Estimation overhead}:
{The main overhead used for estimation $\frac{A_0}{B_0}$ comes from the additional iterations with different sampling pairs ($K_i$, $E_i$). For each pair, $K_i$ clients conduct $R_{i,b}$ rounds of $E_i$ step local iterations for reaching the lower pre-defined loss $F_b$. Hence, for $M$ sampling pairs, the number of iteration $I_{est}$  used for estimation  is
\begin{equation} 
I_{est}=\sum_{i=1}^MK_iE_iR_{i,b}. 
\end{equation}
For training with obtained ($K^*$, $E^*$), then according to (15), 
the total iteration number $I_{tot}$ that used for reaching for reaching $\epsilon=F_R-F^{*}$ precision is 
\begin{equation} 
I_{tot}=K^*E^*R\approx \frac{K^*\left[{A_0 + B_0\left(1+\frac{N-K^\ast}{K^\ast(N-1)}\right)  ({E^\ast})^2}\right]}{\epsilon}. 
\end{equation}
Therefore, the estimation overhead ratio can be expressed as \begin{equation} 
\label{overhead}\frac{I_{est}}{I_{tot}}\approx \epsilon \cdot \frac{ \sum_{i=1}^MK_iE_iR_{i,b}}{K^*\left[{A_0 + B_0\left(1+\frac{N-K^\ast}{K^\ast(N-1)}\right) ({E^\ast})^2}\right]}, 
\end{equation}
where the right hand side of \eqref{overhead} except for $\epsilon$ is bounded by some constant value, and thus the overhead ratio will be marginal for a high error precision with $\epsilon$ being small.}

\section{Design Principles and Solution Properties for Cost Minimization}
\label{sec:property}

 We theoretically analyze the solution properties 
  for different optimization metrics and heterogeneous system parameters. The analysis  not only provides insightful design principles but also gives alternative ways of solving  Problem  \textbf{P3} 
 more efficiently.  Our empirical results in  Section~\ref{sec:experimentation} 
 show that these properties derived for Problem  \textbf{P3} \emph{are still valid for the original} problem  \textbf{P1}. 
 For the ease of presentation, we consider continuous $K$ 
 and $E$ 
 (i.e., before rounding) in this section.

\subsection{Solution Property of $K$ for Minimizing $\tilde{\Expect}[C_\textnormal{tot}]$}
We first show how the normalized price factor $\gamma$ affects the optimal solution of $K$. This provides useful design principles for choosing the number of participants for different optimization metrics, e.g., learning time or energy minimization.   
\begin{theorem}
\label{theorem:t_tot}
For any fixed value of $E$,  $K^{\ast}$ decreases in $\gamma$. In particular, when $\gamma=1$, $\tilde{\Expect}[C_\textnormal{tot}]$ is a strictly increasing function in $K$, thus $K^\ast=1$. 
\end{theorem}
The proof is given in Appendix~\ref{theorem3app}. 
Theorem \ref{theorem:t_tot} shows that sampling fewer devices can reduce the total energy consumption for reaching the target loss  (which would be the main objective as $\gamma$ becomes large). {While this may seem contradictory at the first glance, since a smaller $K$ would result in more rounds to achieve the desired precision $\epsilon$. The key intuition is that the total energy is the sum energy consumption of all selected clients. Although more rounds may cost longer learning time, there are also less number of clients participating in each round, so the total energy consumption can be smaller.} {Particularly, for pure energy minimization task ($\gamma=1$) with $K^*=1$, our proposed sequential ordering scheduling in Theorem~1 reduces to uniformly at random sample one client for transmission in each round, which is  a bit different compared to the Round-Robin scheme, 
because the sequence of our selected clients over several rounds is random not always following the same order.}


Based on Theorem \ref{theorem:t_tot}, 
we derive a trade-off design principle for $K$, generally, with a relatively larger\footnote{We note that  the main difference for reducing time cost between this work and wired FL in \cite{luo2020cost} is that the optimal $K^\ast$ (for $\gamma=0$) in \cite{luo2020cost} is always $N$, where it is usually some value between $[1,N]$ in this paper.  
} $K$ favoring learning time reduction (small $\gamma$) and a smaller $K$  favoring  energy saving (large $\gamma$).  
For a given $\gamma$, the optimal $K^{\ast}$ achieves the best balance between learning time  minimization  and energy consumption  minimization.

Next, we show how heterogeneous system parameters in terms of computation and communication affect the optimal solution of $K$. This shows which directions $K$ and $E$ should change to, when the system environment changes. 
\begin{theorem}
\label{corollary:t_tot_K}
For any fixed value of $E$, when $0<\gamma<1$, $K^\ast$ increases in $t_p$ and decreases in $t_m$, $e_p$ and $e_m$. When $\gamma=0$, $K^\ast$ increases with $\frac{t_p}{t_m}$; whereas when $\gamma=1$, $K^{\ast}$ is independent of $e_m$ and $e_p$. 
\end{theorem}
We give the proof in Appendix~\ref{theorem4app}. 
Intuitively, for general $0\le\gamma<1$,   Theorem~\ref{corollary:t_tot_K} says that when $t_m$ increases (e.g., more clients suffering poor channel conditions or less system bandwidth is available), or $e_m$ and $e_p$ increases (higher average energy costs in communication or computation), 
the optimal strategy to reduce the total cost is to sample fewer clients. Interestingly, when the average computation time $t_p$ increases, we should sample more clients. 
This is because, unlike the per-round communication time accumulated with more clients being sampled, the computation is performed in parallel among the sampled clients, which only slightly increases the per-round time. Sampling more clients can possibly reduce the required total number of rounds, thus bringing down the total learning time. 
When we only focus on energy minimization ($\gamma=1$), the optimal $K$ is independent of $e_m$ and $e_p$, because Theorem \ref{theorem:t_tot} shows that  $K^\ast=1$ in this case.  

\subsection{Solution Property of $E$ for Minimizing $\tilde{\Expect}[C_\textnormal{tot}]$}
We first identify how the normalized price factor $\gamma$ affects the choice of $E$ for different optimization goals, e.g., learning time or energy minimization. 
\begin{theorem}
\label{corollary:t_tot_E1}
For any $0\le\gamma \le1$ and any fixed value of $K$,  as $E$ increases, $\tilde{\Expect}[C_\textnormal{tot}]$ first decreases and then increases. 
\end{theorem}
The proof is given in Appendix~\ref{theorem5app}. Theorem~\ref{corollary:t_tot_E1} shows that, unlike the strategy for minimizing energy consumption where $K^\ast$ lies at the boundary, the optimal $E$ should not be set too small nor too large no matter for saving energy consumption or learning time.

Then, the following theorem characterizes how heterogeneous system parameters in computation and communication affect the design principle of optimal $E$. 

\begin{theorem}
\label{corollary:t_tot_E2}
For any fixed value of $E$, when $0<\gamma<1$, $E^\ast$ increases as $t_p$ or $e_p$ decreases. $E^*$ increases in $\frac{t_m}{t_p}$ when $\gamma=0$, and $E^*$ increases in $\frac{e_m}{e_p}$ when $\gamma=1$. 
\end{theorem}
We give the proof in Appendix~\ref{theorem6app}. 
 Theorem~\ref{corollary:t_tot_E2} says that for any given $K$, when or $t_p$ or $e_p$ decreases,  
the optimal strategy for learning time minimization ($\gamma=0$) or energy consumption  minimization ($\gamma=1$)  is to perform more steps of iterations (i.e., increase $E$) before aggregation. {This is because when computation is cheaper or when communication is more expensive, it is beneficial to perform more computation in each round, as intuition suggests. This theorem provides theoretical evidence for the empirical observations  in \cite{mcmahan2017communication, yu2018parallel, stich2018local,wang2018adaptive}.}

\begin{figure}[!t]
	\centering
	\includegraphics[width=8.5cm,height=5.2cm]{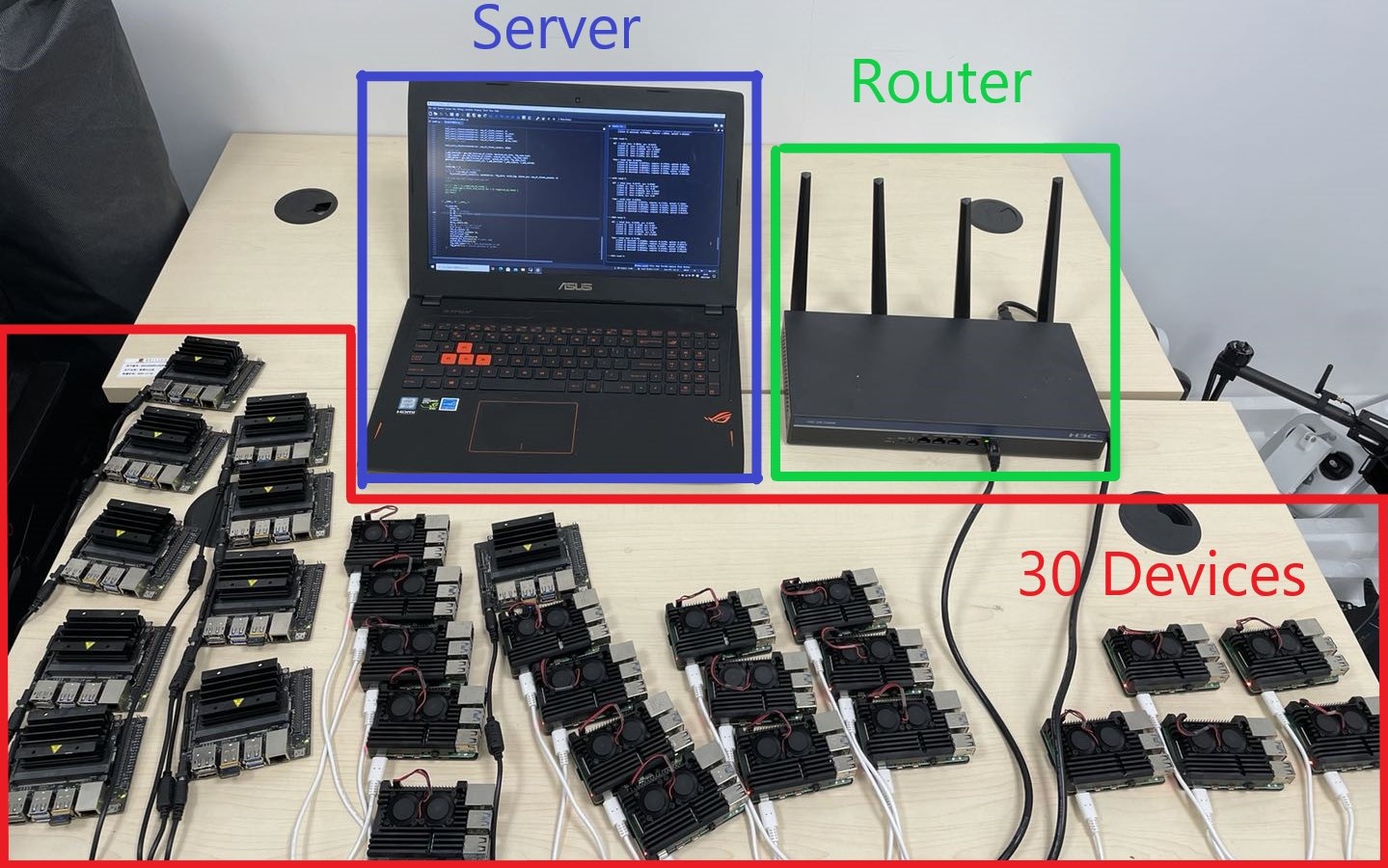}
	\caption{Hardware prototype with the laptop being central server, 20 Raspberry Pi and 10 Jetson Nano being devices. During the FL experiments, the wireless router is placed 2 meters away from all the devices.}
\end{figure}

\section{Experimental Evaluation}
\label{sec:experimentation}
In this section, we evaluate the performance of our proposed scheduling scheme, cost-effective FL algorithm and derived solution properties. 
 We start by presenting the evaluation setup, and then show the experimental results. 

\subsection{Experimental Setup}
\subsubsection{Platforms}
{We conducted experiments both on a networked hardware prototype system and in a simulated environment. 
Our prototype system, as illustrated in Fig.~2,
consists of 30 edge devices with $20$ Raspberry Pis (version 4) and $10$ Jetson Nanos as well as a laptop computer serving as the central server. All devices are interconnected via an enterprise Wi-Fi router, and we developed a TCP-based socket interface for the peer-to-peer connection.\footnote{{To characterize LTE-based wireless networks as in Fig.~1, we manually change the original communication protocol in our TS-based WiFi system, so that all devices start their local training at the same time and without calculating the downlink time.}} 
In the simulation system, we simulated a 
cellular network of an urban microcell consisting of a central BS serving as the server and $N=100$  clients. The BS and server were co-located at the center of the cell with a radius of $2$ km, and the clients were uniformly distributed in the cell.
}
\subsubsection{Datasets and Models} We evaluate our results both on a real dataset and a synthetic dataset. {For the real dataset, following a same setup of \cite{li2020federated}, we adopted the widely used MNIST dataset and EMNIST dataset, which contains gray-scale images of  handwritten digits and characters.} 
For the synthetic dataset, we follow a similar setup to that in \cite{li2018federated}, which generates $60$-dimensional random vectors as input data. The synthetic data is denoted by $Synthetic \  (\alpha, \beta)$ with $\alpha$ and $\beta$ representing the statistical heterogeneity (i.e., how non-i.i.d. the data are). 
We adopt the \emph{convex} {multinomial logistic regression} model for both datasets with model size around $0.03$ MB \cite{li2019convergence}.

\subsubsection{Implementation}
{Based on the above, we consider the following three experimental setups with heterogeneous data distribution}.

\emph{Prototype Setup}: We conduct the first experiment on the prototype system using  
MNIST dataset, where we divide $9,000$ data samples among the $N\!=\!30$ edge devices (20 Raspberry Pis and 10 Jetson Nanos) in a non-i.i.d. fashion, with each device containing a balanced number of $300$ samples of only $2$ digit labels. 

 {\emph{Simulation Setup 1}: We conduct the second experiment in the simulated system using  
     EMNIST dataset, where we divide 48,000 image  samples among $N=100$ clients in a non-i.i.d. fashion, with each client containing a balanced number of 480 samples of only 2 classes. }
    
    \emph{Simulation Setup 2}: We conduct the third experiment in the simulated system using  
    $Synthetic \ (1, 1)$ dataset for statistical heterogeneity. 
We generate $24,517$ data samples and distribute them among $N\!=\!100$ mobile devices in an unbalanced power law distribution,  where the number of samples in each device has a mean of $245$ and standard deviation of $362$.

\subsubsection{Training  Parameters}  For all experiments, we initialize our model with $\mathbf{w}_0=\mathbf{0}$ and SGD batch size $b=64$. In each round, we uniformly sample $K$ devices at random, which run $E$ steps of SGD in parallel.  For all experiments, we use an initial learning rate $\eta_0=0.1$ with decay rate $\frac{{\eta}_0}{1+r}$, where $r$ is communication round index.  We evaluate the aggregated model in each round on the global loss function. 
 Each result is averaged over 50 experiments.
\subsubsection{Heterogeneous System Parameters}
The prototype system allows us to capture real system heterogeneity in terms of communication and computation time, which we measured the average  $t_p = 4.9 \times 10^{-3}$ s with standard deviation $1.43 \times 10^{-3}$ s and $t_m = 0.16$~s with standard deviation $0.03$ s. 
We do not consider the energy cost in the prototype system because it is difficult to measure. 
For the simulated mobile edge system, we   
assume the system bandwidth is 1.8 MHz 
and adopt the standard LTE communications model based on a well-known urban channel model with the mean    per-client throughput  $1.4$~Mbit/s (thus $t_m\approx 0.2$~s) as in \cite{nishio2019client}. 
We generate the computation time and energy consumption for each client using a truncated normal distribution with
mean value of $t_p\!=\!0.5$~s (emulating slow devices), 
$e_p=0.01$~J, and $e_m = 0.02$~J . 
 {According to the definition of $\gamma$, we  
unify the time and energy costs such that one second is equivalent to $1-\gamma$ dollars~(\$) and one Joule is equivalent to $\gamma$ dollars~(\$).} 

\subsection{{Performance Results}}
In this subsection, we first evaluate the performance of our proposed TS scheduling strategy with two existing benchmarks. Then, we compare the performance of the proposed solution ($K^*$, $E^*$)  
 obtained from Algorithm~2 for solving Problem \textbf{P3}, 
 with that 
of the empirical optimal solution ($K_\textnormal{OPT}$, $E_\textnormal{OPT}$) achieved by an exhaustive search on 
the optimal solution of the original problem \textbf{P1}. Finally, we validate our derived design principles and solution properties.
\begin{figure*}[!t]
\centering
\subfigure[Loss with scheduling schemes
]{\label{schedul_lossa}\includegraphics[width=5.4cm,height=4.5cm]{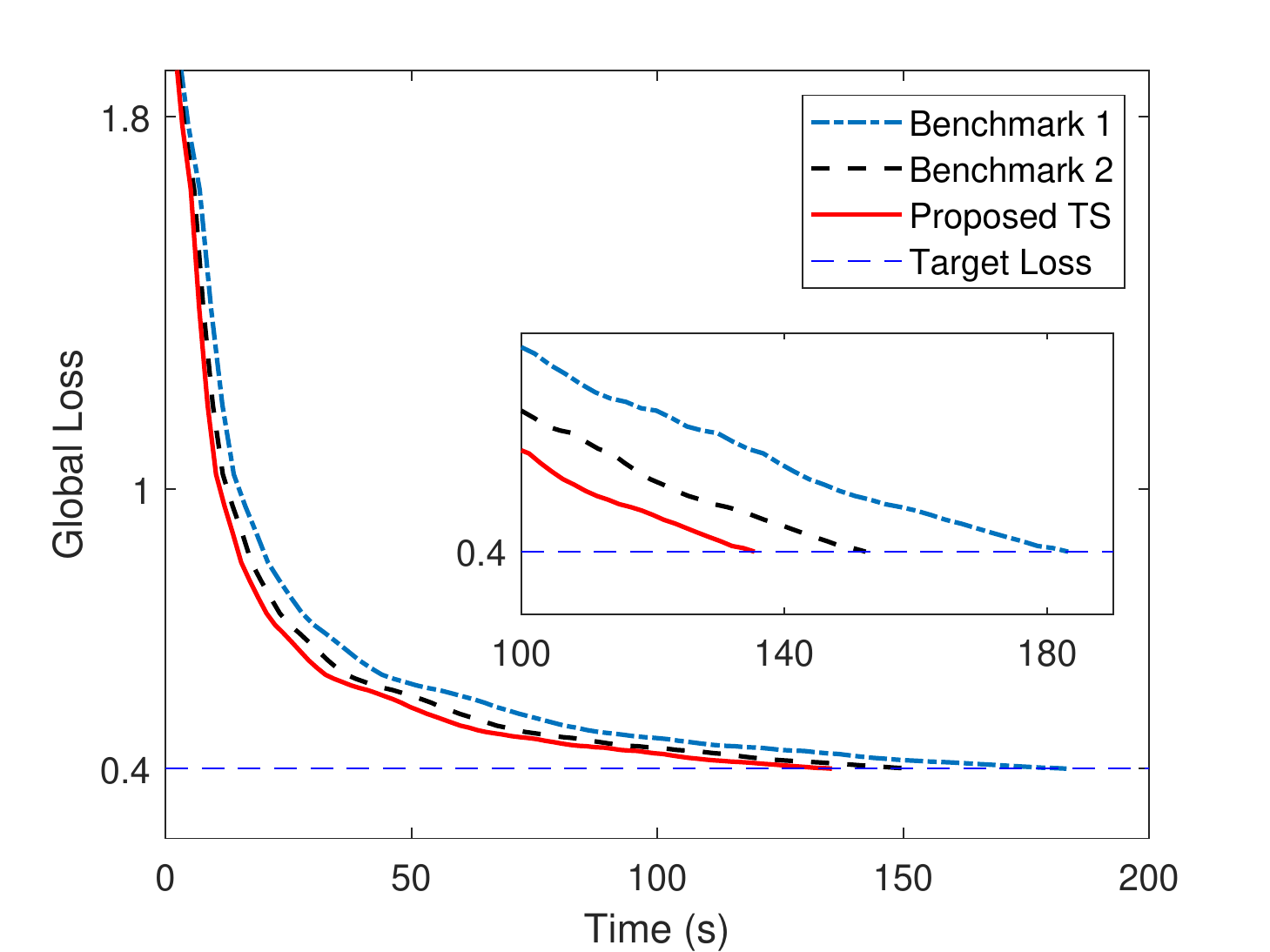}}
\subfigure[$T_\textnormal{tot}$ with scheduling schemes and $E$]{\label{schedule_lossc}\includegraphics[width=5.4cm,height=4.5cm]{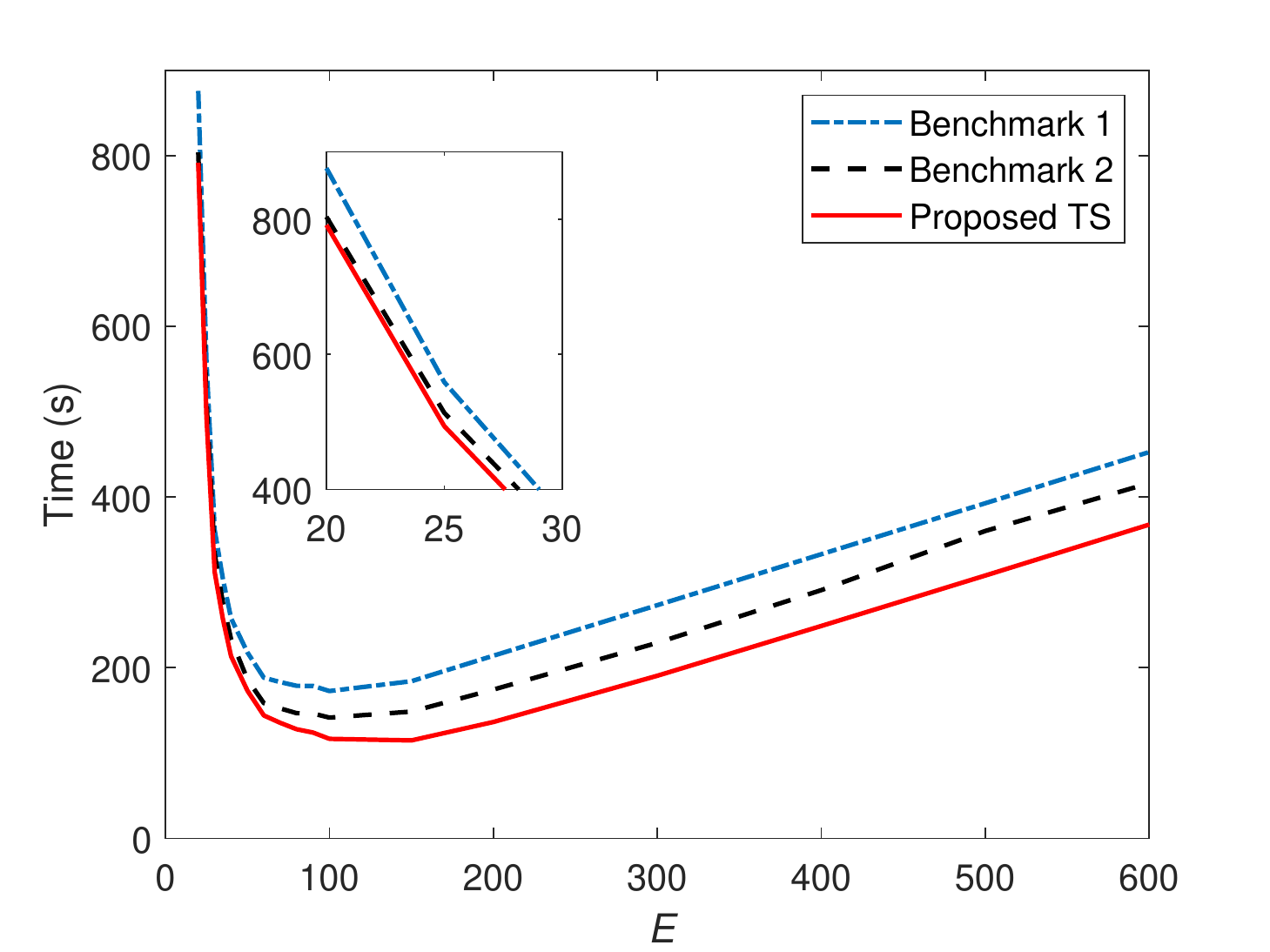}}
\subfigure[$T_\textnormal{tot}$ with scheduling schemes and $K$]{\label{schedule_t_tot}\includegraphics[width=5.4cm,height=4.5cm]{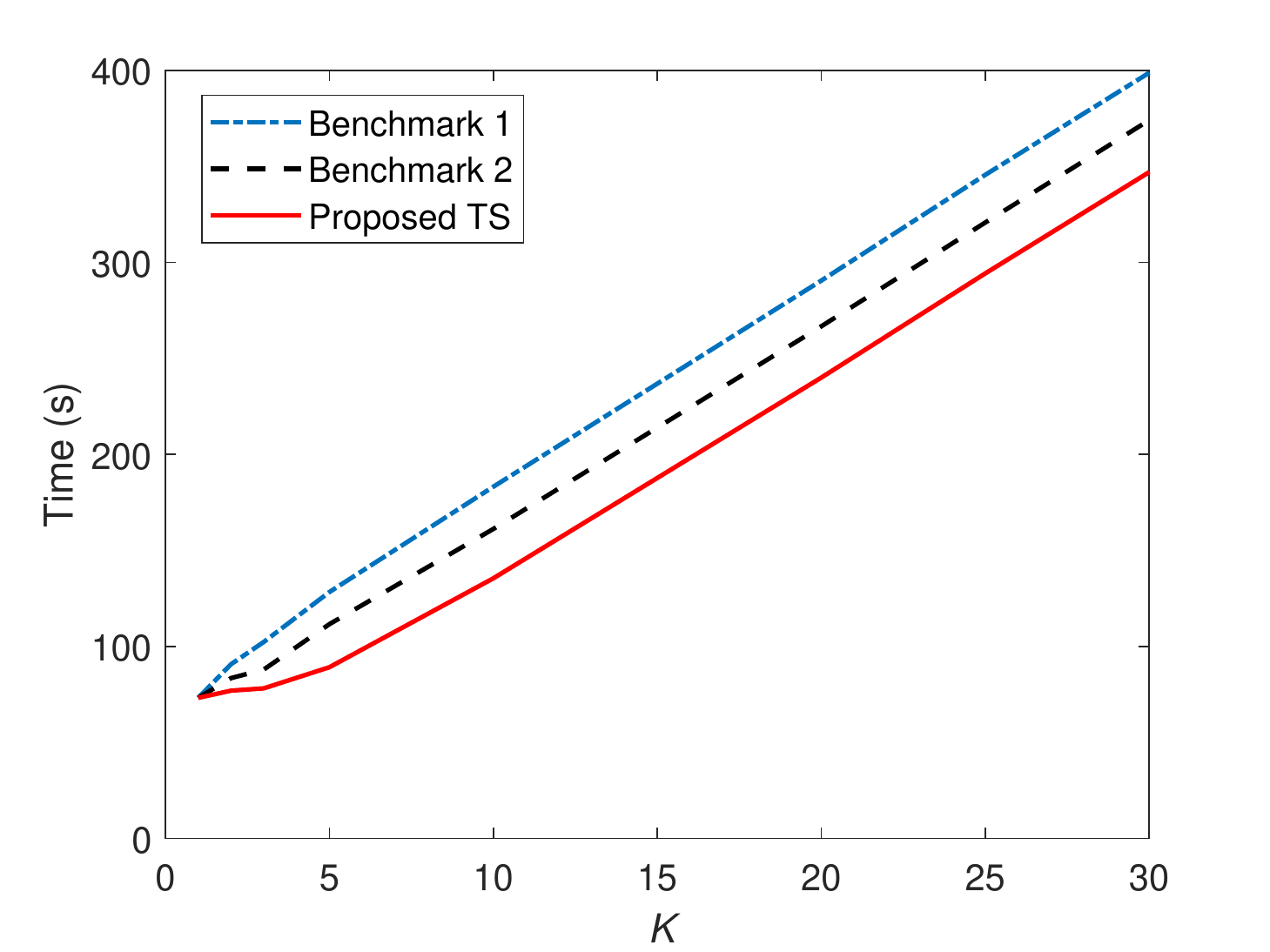}}
\caption{
Total learning time performance of \textbf{Prototype Setup} for different scheduling strategies with logistic regression and MNIST dataset. (a): For a given pair $\left(K=10, E=70\right)$, our proposed scheduling strategy achieves the target loss 0.4 faster than the two benchmark schemes. (b): For fixed $K=10$, our proposed scheme achieve the target loss using shorter time for different $E$. 
(c): For fixed $E=70$, our proposed scheme achieve the target loss using shorter time for different $K$.} 
\label{scheduling}
\end{figure*}

\subsubsection{Scheduling Performance}
Fig.~\ref{scheduling} 
compares the total learning time \textbf{performance of our proposed TS scheduling with the following two benchmarks} in Prototype Setup: 

\begin{itemize}
    \item Benchmark 1: existing TS scheduling strategy in \cite{tran2019federated}, where communication does not begin until all sampled clients complete their local computations.
    \item Benchmark 2: existing FS scheduling strategy in \cite{shi2020device, chen2020convergence,yang2019energy}, where the algorithm allocates the bandwidth at the start of each round for the sampled clients, then the bandwidth allocation remains unchanged throughout the round and every selected client starts its communication using its allocated bandwidth whenever its computation has finished.
\end{itemize}

 
 The main observations are as follows.
\begin{itemize}
  \item Fig.~\ref{scheduling}  demonstrates that \emph{our proposed scheduling strategy  in Theorem \ref{theorem1} 
  achieves the target loss faster than the two benchmarks}, since our method can better utilize system resources and adapt to both computation and communication heterogeneity. 
  \item \emph{The performance gaps between our scheme and the two benchmarks increase in $E$ at first and then remain stable in Fig.~\ref{scheduling}(b)}. The reason for this observation is that, as $E$ increases, the computation times  between different sampled clients are more different,  allowing our scheme to {better utilize the full bandwidth} and thus save per-round time. When $E$ is large enough, the per-round time is only {dominated by computation} time, hence the impact of different scheduling is insignificant as scheduling mainly affects the communication time. 
 
  \item We also observe that \emph{the performance gaps between our scheme and the two benchmarks increase in $K$ at first and then remain stable in Fig.~\ref{scheduling}(c)}. The explanation for this observation 
  is that as $K$ increases, 
  the difference between the maximum and minimum computation time among the selected $K$ clients can be larger, allowing our scheme to {better utilize the computational heterogeneity}, e.g., communication could begin at an early stage instead of waiting for the rest clients. When $K$ is large enough, the per-round time is mainly {dominated by communication} time, and the gain from computational heterogeneity tends to be stable between different scheduling schemes.
 \end{itemize}

\begin{table*}[!t]
 \caption{Number of rounds for reaching estimation loss $F_a$ and $F_b$ for estimation of $\frac{A_0}{B_0}$ for three Setups}
 \label{sample3}
  \centering
  {
\begin{tabular}{c||c||c|c|c|c|c|c||c}
\toprule[1.2pt]
\multirow{3}{*}{\makecell[c]{Prototype Setup\\MNIST}}
& \multirow{3}{*}{\makecell[c]{Estimation loss\\$F_a=0.65$\\\!$F_b=0.55$}}    & \scriptsize{Samples of $\left(K,E\right)$} & $\left(1, 30\right)$ &  $\left(5, 80\right)$ & $\left(10, 40\right)$ & $\left(15, 100\right)$ &$\left(20, 50\right)$  &  \multirow{3}{*}{{\makecell[c]{\textbf{Estimated}\\$\frac{A_0}{B_0}$=36,500}}  } \\ \cline{3-8}
                       &                       & \scriptsize{Rounds to achieve $F_a$} &  50 & 17 & 17 & 13 & 14 &                \\ \cline{3-8}
                       &                       & \scriptsize{Rounds to achieve $F_b$} &   75 & 32 & 30 & 21 & 25  &                 \\ \midrule[1.2pt]

\multirow{3}{*}{\makecell[c]{Simulation Setup 1\\EMNIST }}    & \multirow{3}{*}{\makecell[c]{Estimation loss\\$F_a=1.2$\\$F_b=1.0$}}    & Samples of $\left(K,E\right)$&    $\left(5, 7\right)$&    $\left(10, 10\right)$ &  $\left(20, 20\right)$ & $\left(30, 30\right)$ & $\left(40, 40\right)$  &    \multirow{3}{*}{{\makecell[c]{\textbf{Estimated}\\$\frac{A_0}{B_0}$=960}}  } \\ \cline{3-8}
                           &                       & Rounds to achieve $F_a$  &  17&13&10& 11 &14  &                \\ \cline{3-8}
                       &                       & Rounds to achieve $F_b$  & 29& 20 & 17 & 18& 28 &               \\ \bottomrule[1.2pt]
\multirow{3}{*}{\makecell[c]{Simulation Setup 2\\Synthetic (1,1) }}    & \multirow{3}{*}{\makecell[c]{Estimation loss\\$F_a=1.7$\\$F_b=1.5$}}    & Samples of $\left(K,E\right)$&    $\left(5, 7\right)$&    $\left(10, 10\right)$ &  $\left(20, 20\right)$ & $\left(30, 30\right)$ & $\left(40, 40\right)$  &    \multirow{3}{*}{{\makecell[c]{\textbf{Estimated}\\$\frac{A_0}{B_0}$=1850}}  } \\ \cline{3-8}
                           &                       & Rounds to achieve $F_a$  &  41&28&22& 19 &18  &                \\ \cline{3-8}
                       &                       & Rounds to achieve $F_b$  & 78& 52 & 39 & 34& 31 &               \\ \bottomrule[1.2pt]
\end{tabular}}
\end{table*}
\begin{figure*}[!t]
\centering
\subfigure[Loss with different 
$K$]{\label{hd_lossa}\includegraphics[width=5.4cm,height=4.5cm]{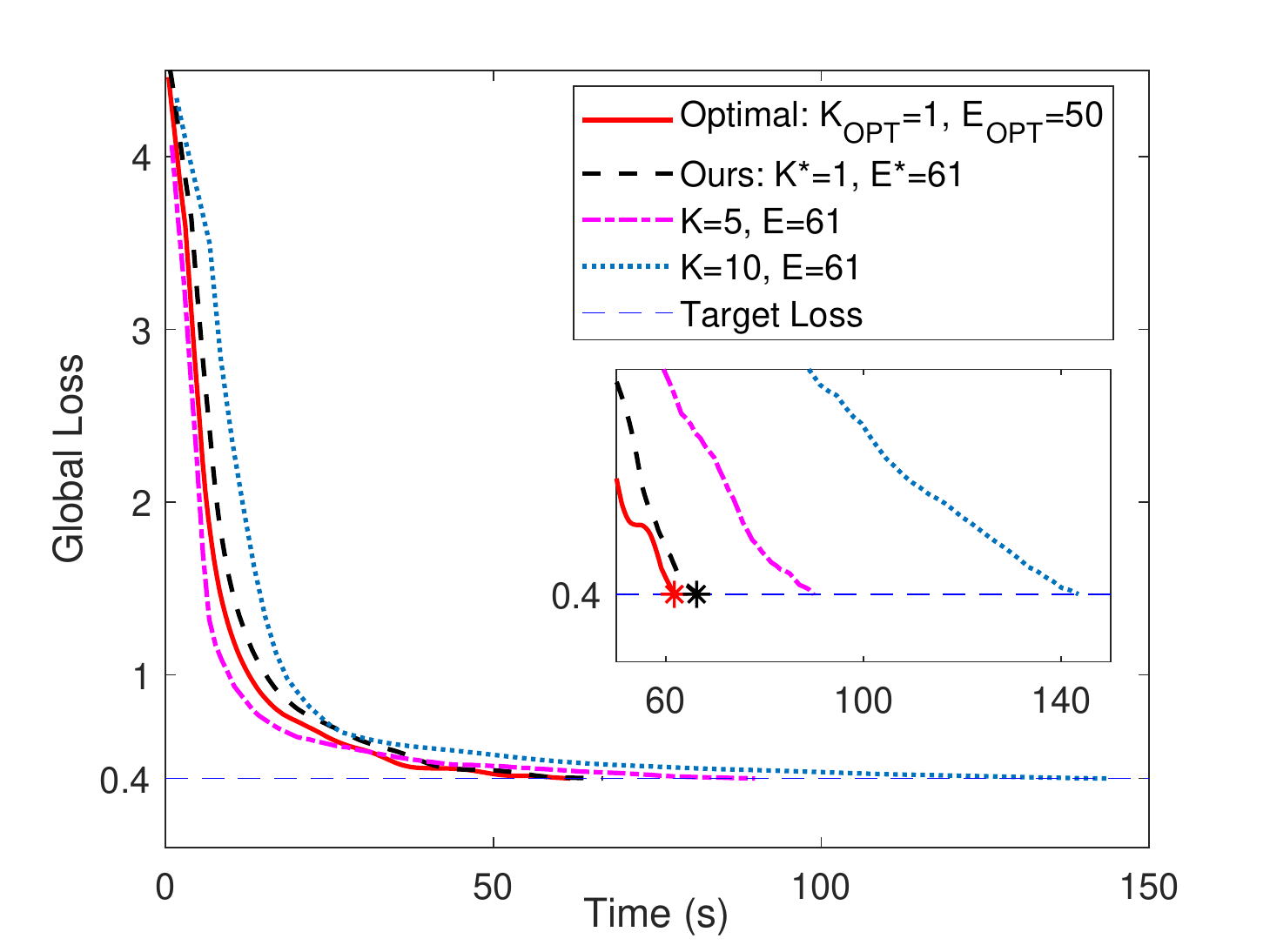}}
\subfigure[Loss with different 
$E$]{\label{hd_lossc}\includegraphics[width=5.4cm,height=4.5cm]{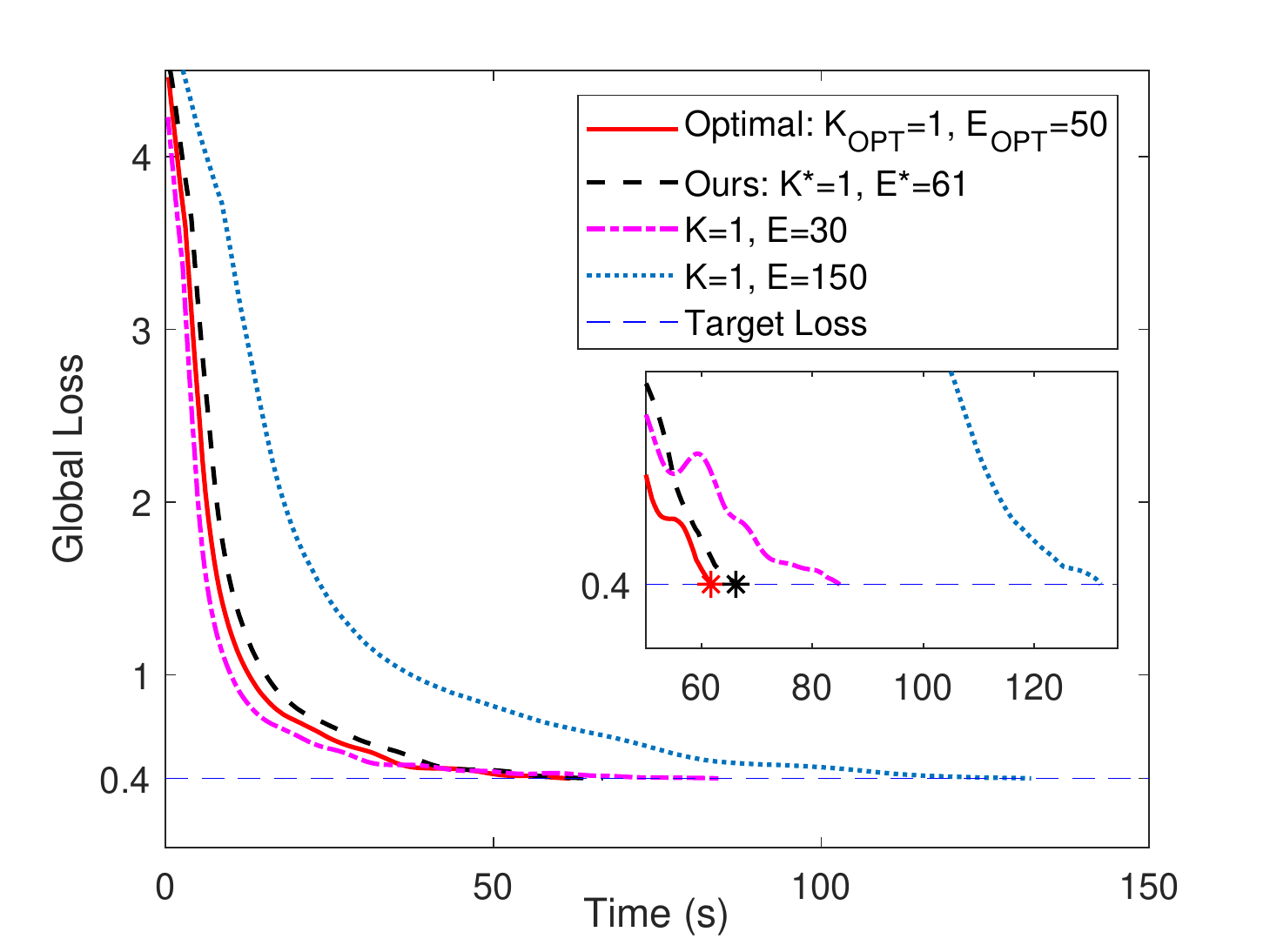}}
\subfigure[$T_\textnormal{tot}$ with different $\left(K, E\right)$ ]{\label{hd_t_tot}\includegraphics[width=5.4cm,height=4.5cm]{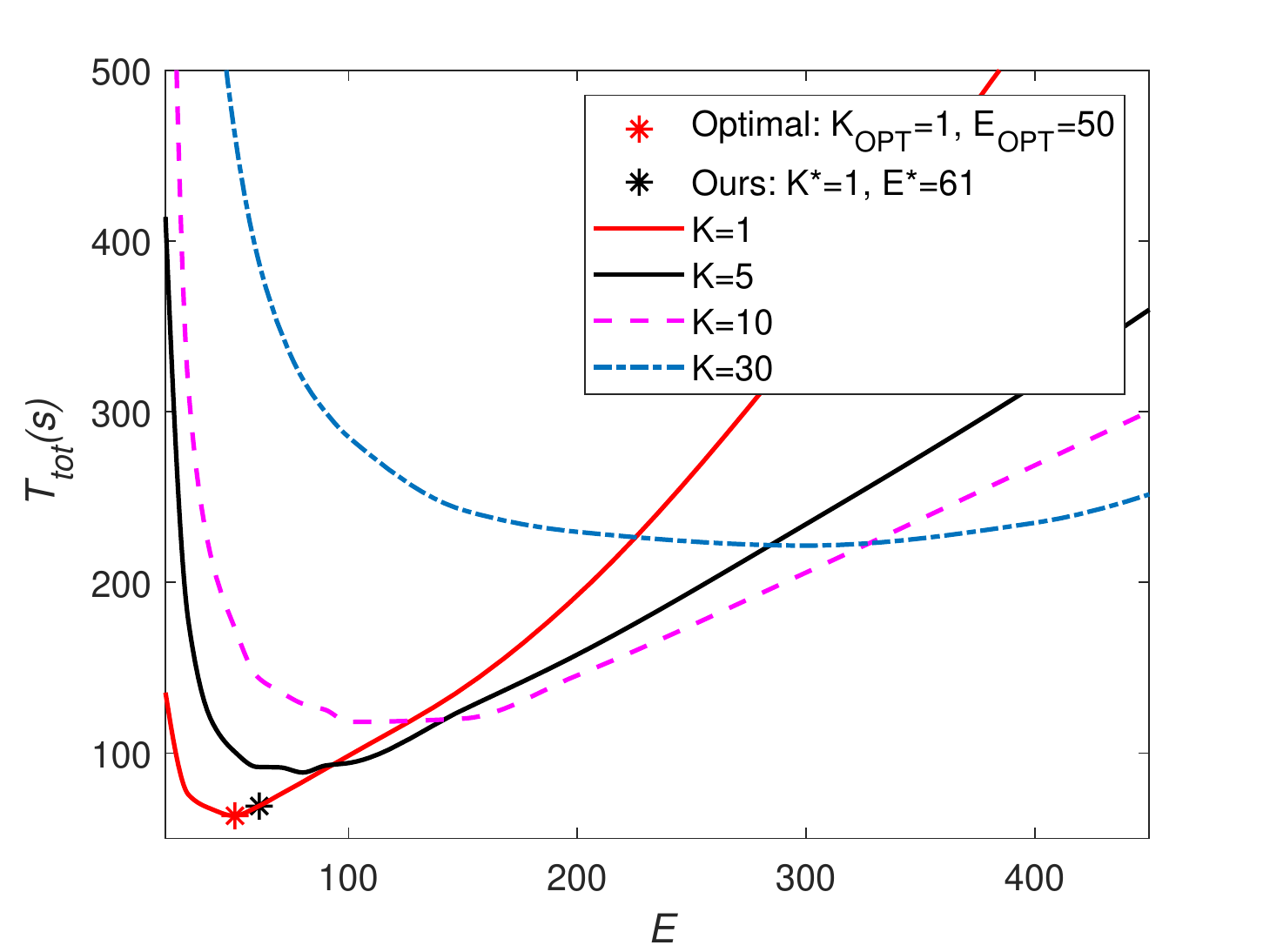}}
\subfigure[Accuracy with different $E$ ]{\label{hd_lossb}\includegraphics[width=5.4cm,height=4.5cm]{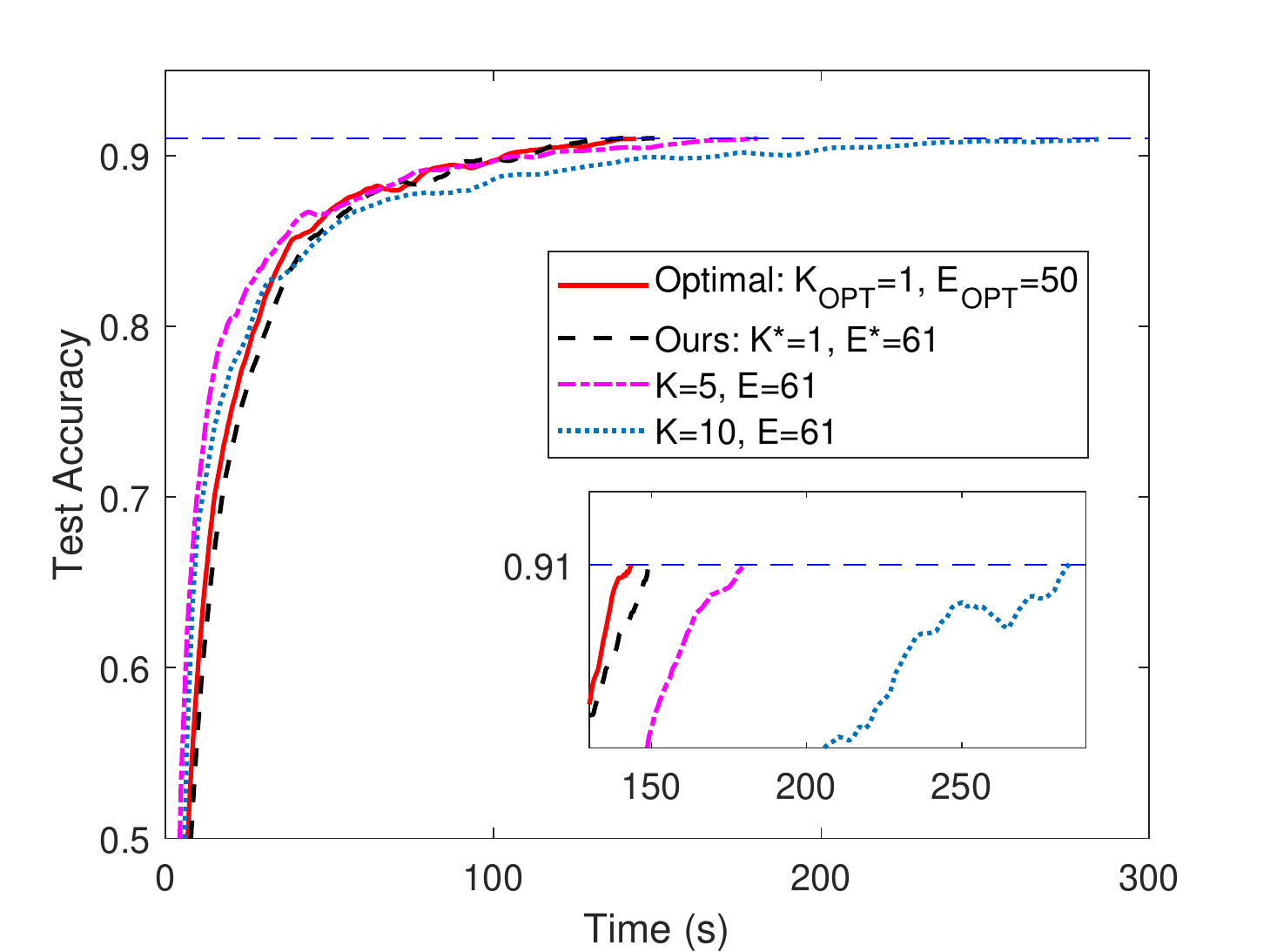}}
\subfigure[Accuracy with different $K$ ]{\label{hd_acc_t}\includegraphics[width=5.4cm,height=4.5cm]{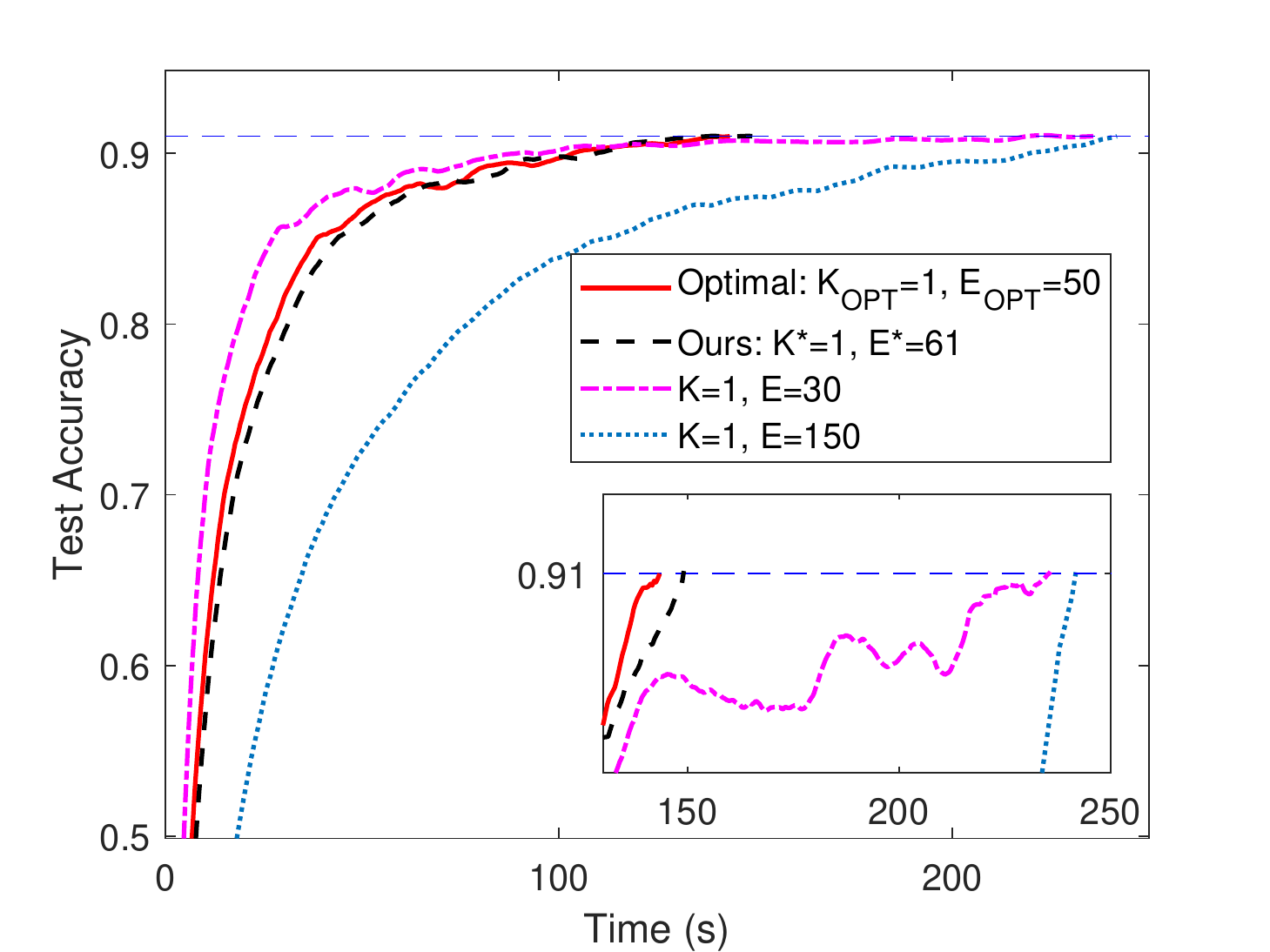}}
\caption{
Training performance of $T_\textnormal{tot}$ for  \textbf{Prototype Setup} with logistic regression and MNIST dataset for $\gamma\!=\!0$. (a)-(c): Our solution achieves the target loss using $66.2$~s compared to the empirical optimum $61.7$~s, 
but faster than 
those with $E$ being too small or too large 
and those with $K$ being large. (d)-(e): 
 Our solution achieves target test accuracy $91\%$  slightly longer than the optimal solution, but faster than 
 the non-optimal values of $\left(K,E\right)$ in (a) and (b).}
\label{hd}
\end{figure*}

\subsubsection{Optimality Performance}
We first present the estimation process and results of $\frac{A_0}{B_0}$  for the three  experiment setups in Table~\ref{sample3}.  Figs.~\ref{hd}--\ref{c_tot} compare the \textbf{performance of our {proposed solution of ($K^*,\ E^*$) with empirical optimal solution}  ($K_\textnormal{OPT},\ E_\textnormal{OPT}$)} {and other} ($K$, $E$) {pairs}\footnote{The benchmark $E$ and $K$ are chosen to be either lager or smaller than the empirical optimal ones for comparison.} for Prototype and  Simulation Setups, respectively.\footnote{We note that ($K^*$, $E^*$) are obtained by estimating the average value of $\frac{A_0}{B_0}$, whose estimation process is summarized in Table~\ref{sample3}. Specifically, we empirically  set  two relatively high target losses $F_a$ and $F_b$ with  a few sampling pairs of $\left(K,E\right)$. Note that due to different learning tasks and statistical heterogeneity, the sampling range of $E$ in Prototype Setup is larger than that in Simulation Setup. Then, we record the corresponding number of rounds for reaching $F_a$ and $F_b$, based on which  we calculate the averaged estimation value of $\frac{A_0}{B_0}$ using  \eqref{A0B05}.} The key observations are as follows. 

\begin{figure*}[!t]
\centering
\subfigure[Loss with different 
$K$]{\label{setup_femnist0}\includegraphics[width=5.4cm,height=4.5cm]{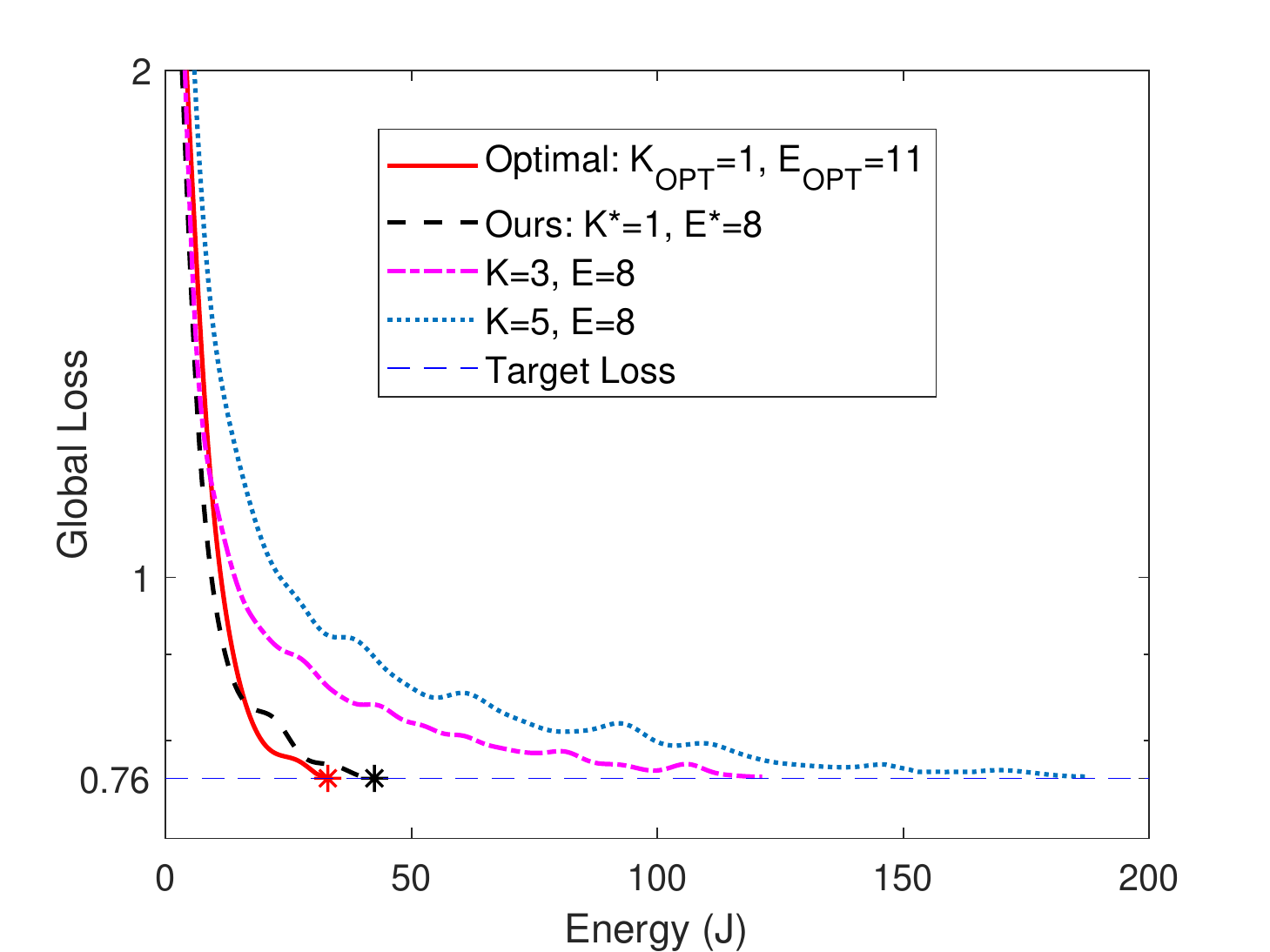}}
\subfigure[Loss with different 
$E$]{\label{setup_femnist1}\includegraphics[width=5.4cm,height=4.5cm]{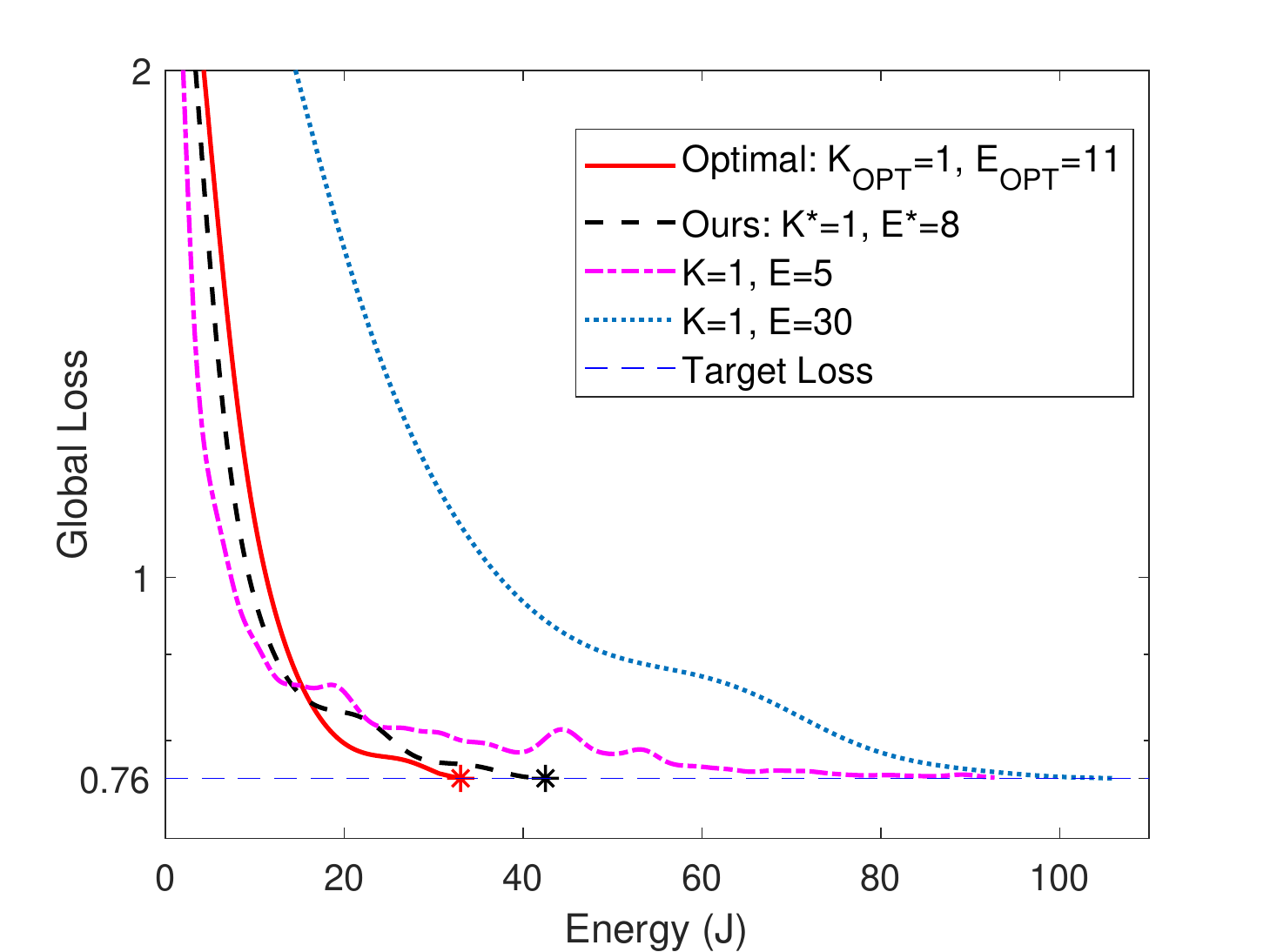}}
\subfigure[$e_\textnormal{tot}$ with different $\left(K, E\right)$ ]{\label{setup_femnist2}\includegraphics[width=5.4cm,height=4.5cm]{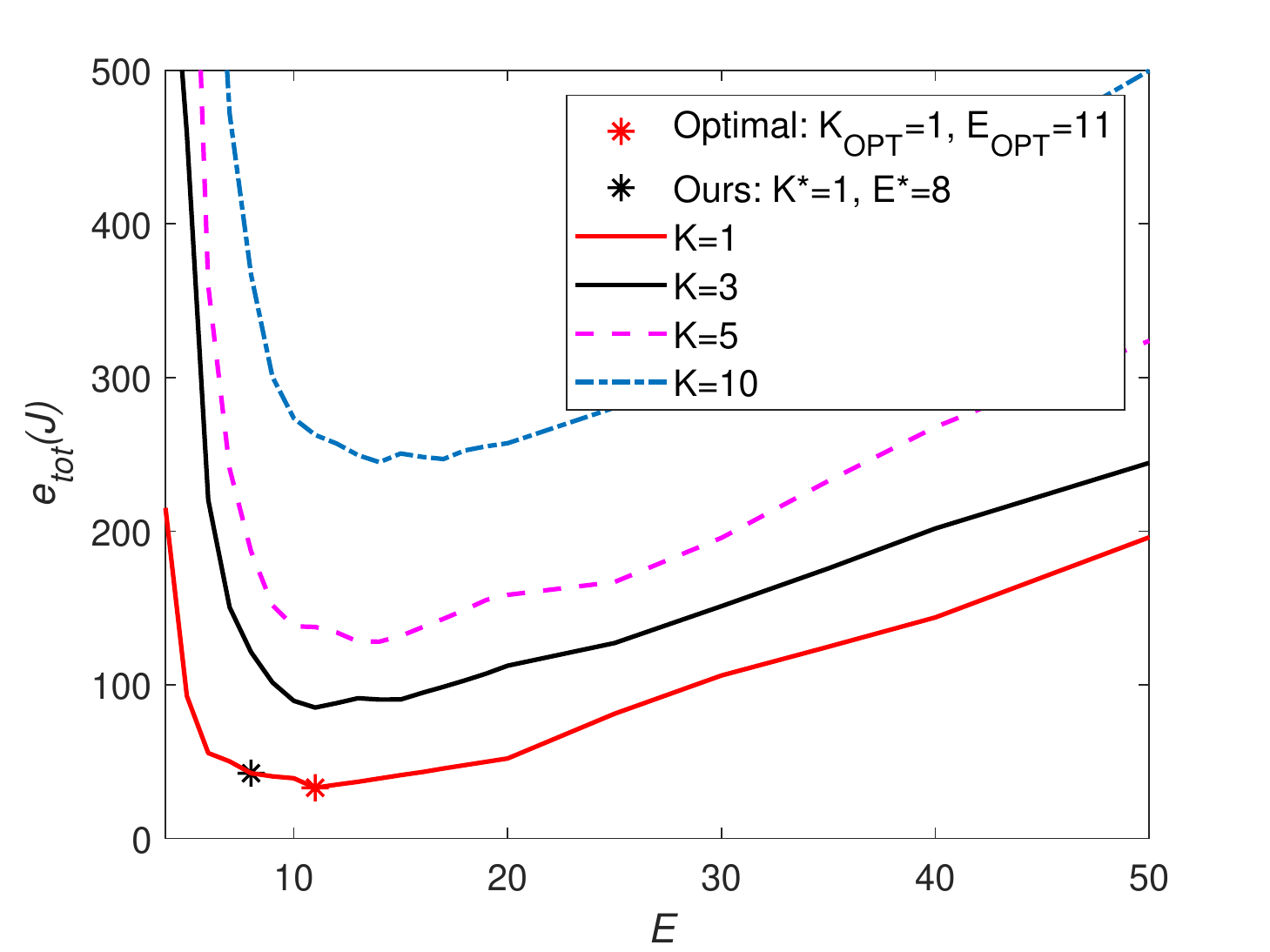}}
\subfigure[Accuracy with different $E$ ]{\label{setup_femnist3}\includegraphics[width=5.4cm,height=4.5cm]{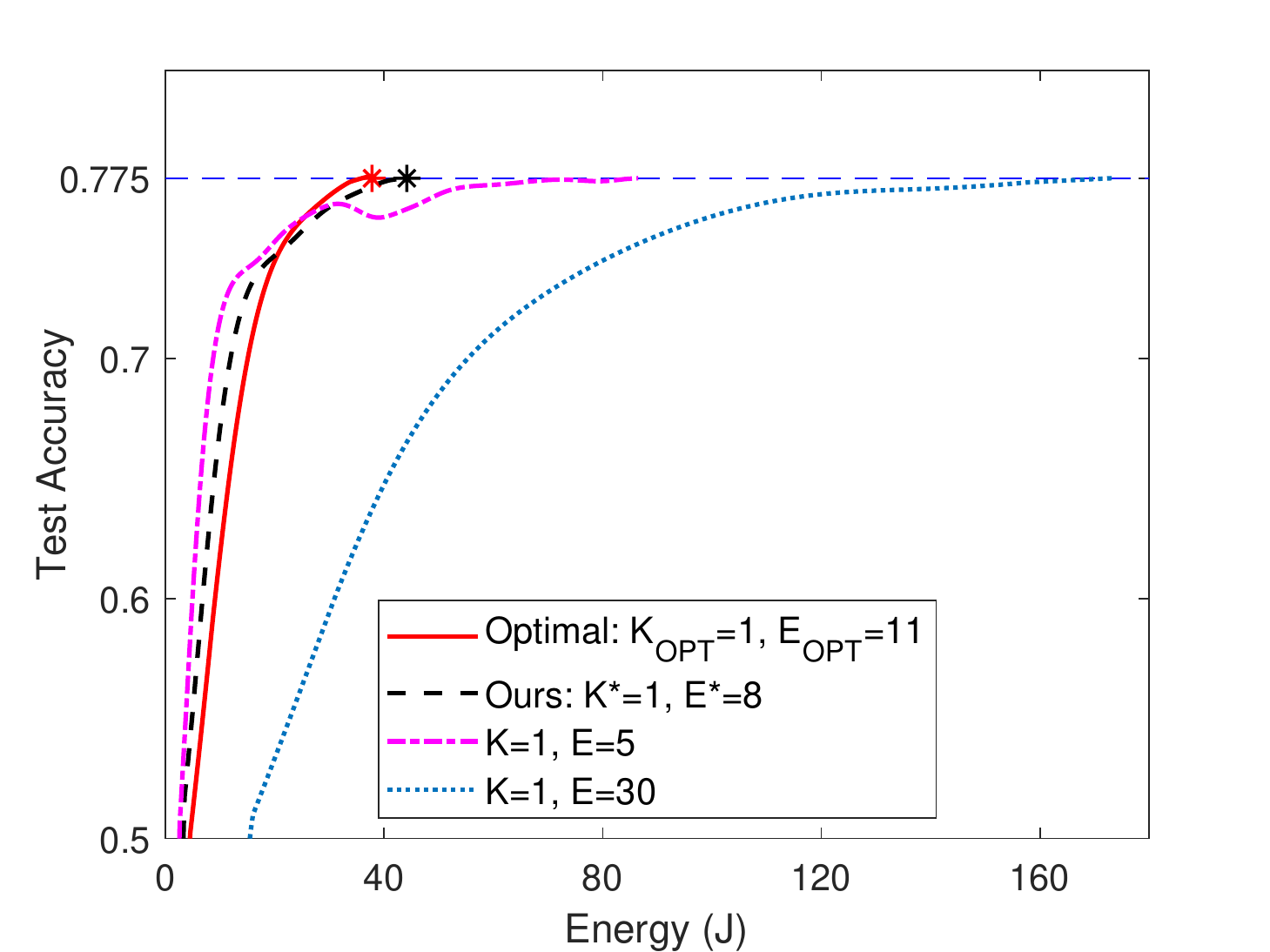}}
\subfigure[Accuracy with different $K$ ]{\label{setup_femnist4}\includegraphics[width=5.4cm,height=4.5cm]{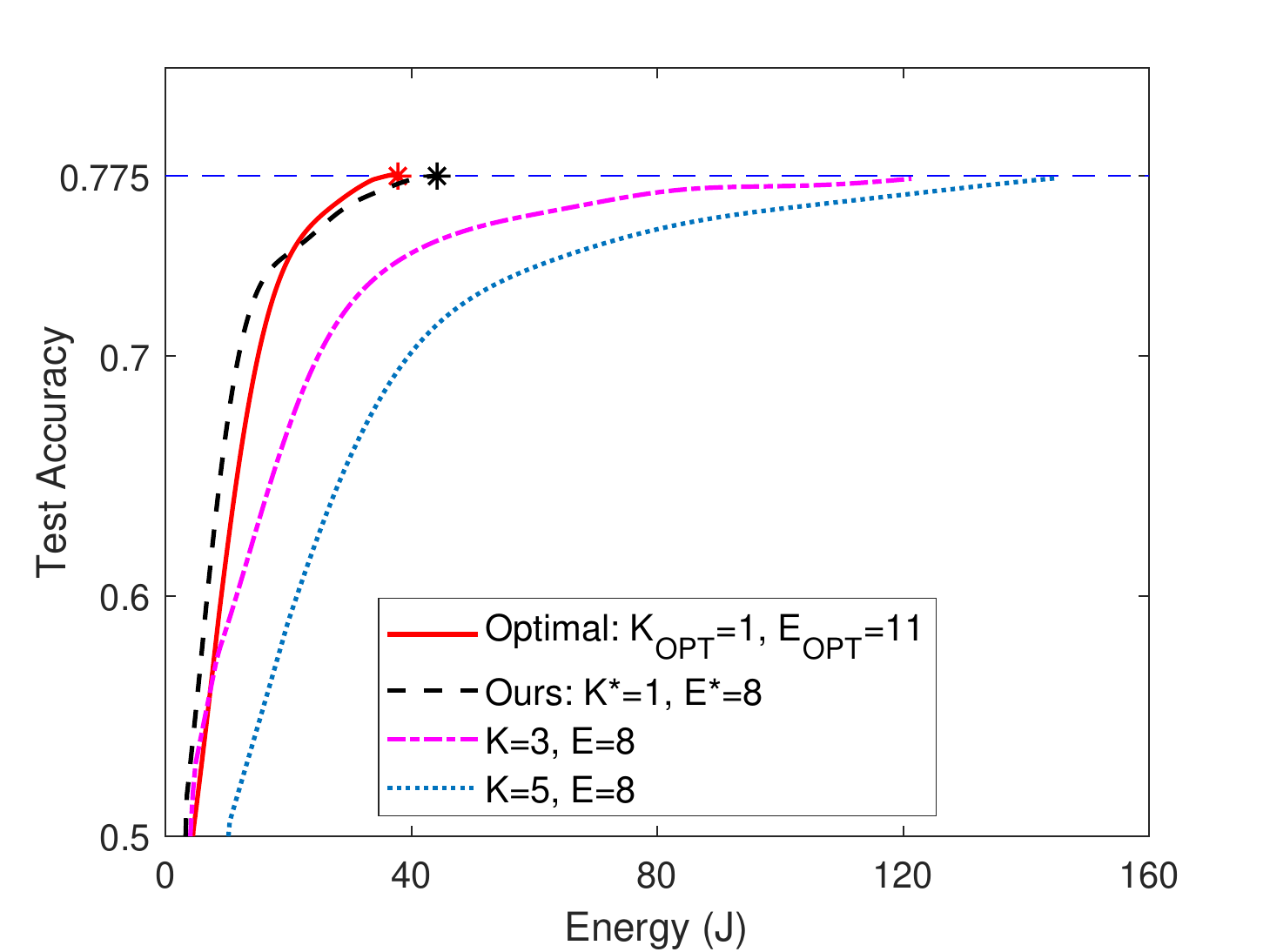}}
\caption{
{Training performance of $e_\textnormal{tot}$ for \textbf{Simulation Setup 1} with logistic regression and EMNIST dataset for $\gamma\!=\!1$. (a)-(c): Our solution achieves the target loss using $42.4$~J, which slightly larger than the empirical optimum $33.7$~J,  
but saves several folds of energy compared to  
those ($E$,$K$) pairs without optimization. 
(d-e): 
 Our solution achieves the target test accuracy $77.5\%$  using slightly more energy than the optimal solution, but saves several folds of energy than 
 the non-optimal values of $\left(K,E\right)$ in (a) and (b).}}
\label{emnist}
\end{figure*}

\begin{itemize}
    \item 
Fig.~\ref{hd} 
shows the {learning time cost} $T_\textnormal{tot}$ 
for reaching the target loss under different  $\left(K, E\right)$  for Prototype Setup.
\footnote{For hardware prototype, we only show the convergence performance with $T_\textnormal{tot}$ 
for $\gamma = 0$.}  
{In particular, \emph{our solution achieves the target loss using $66.2$ s compared to the empirical optimum $61.7$ s}, whereas other ($K$, $E$) pairs {without
optimization may consume a learning time that is several folds more}, as shown in Fig.~\ref{hd}(c).\footnote{The intersections in Fig. 4(c) and Fig. 5(d) show that, for a given $E$, different numbers of $K$ may result in the same total learning time. This is because the trade-off between the total number of rounds for reaching the target loss and the per-round time, as a large $K$ reduces the number of rounds but yields a longer per-round time, whereas
a small $K$ reduces the per-round time but requires more rounds.}}   
 \item {Fig.~\ref{emnist} depicts the energy cost $e_\textnormal{tot}$ for reaching the target loss under different  $\left(K, E\right)$  for Simulation Setup 1.  \emph{Our solution  achieves the target loss and accuracy with the similar energy cost compared to the empirical optimal solution,} and saves several folds of energy compared to those ($K$, $E$) pairs without
optimization.} 

\item Figs.~\ref{c_tot} shows the {total cost $C_\textnormal{tot}$} for reaching the target loss for {different $\gamma$} under different $\left(K,\ E\right)$ for Simulation Setup 2.  
Comparing to other ($K$, $E$) pairs without
optimization in Fig.~\ref{c_tot}(c)-(f), {our proposed solutions incur a  similar total cost as the corresponding empirical optimal ones through the entire range of $\gamma$}. 
\end{itemize}

\begin{figure*}[t]
\centering
\subfigure[Loss with $E$ for $\gamma=0.5$]{
\includegraphics[width=5.34cm,height=4.4cm]{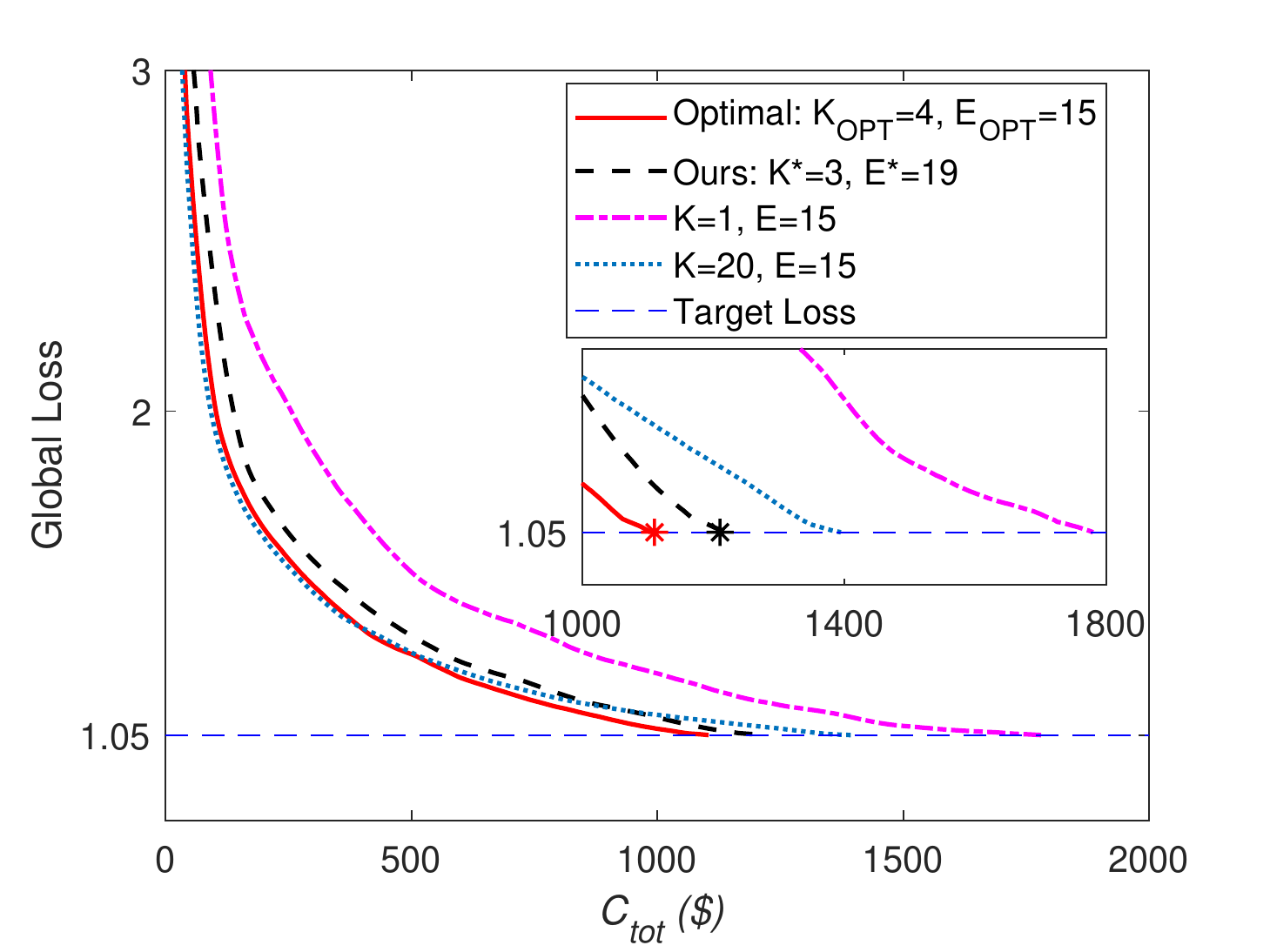}}
\subfigure[Loss with $K$ for $\gamma=0.5$]{\includegraphics[width=5.27cm,height=4.4cm]{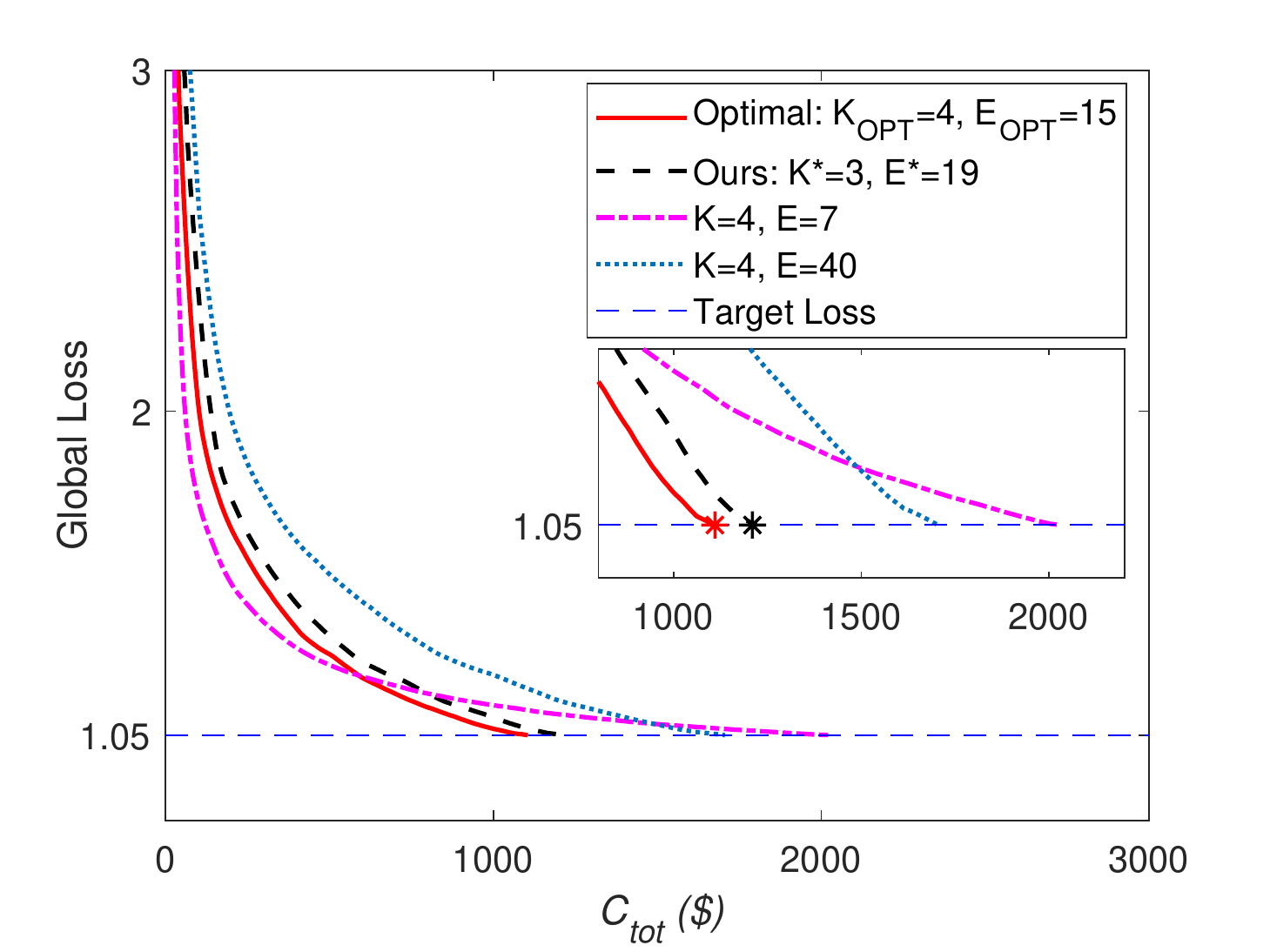}}
\subfigure[\label{ctot_gamma05}$C_\textnormal{tot}$ with $\left(K, E\right)$ for $\gamma=0.5$]{
\includegraphics[width=5.27cm,height=4.4cm]{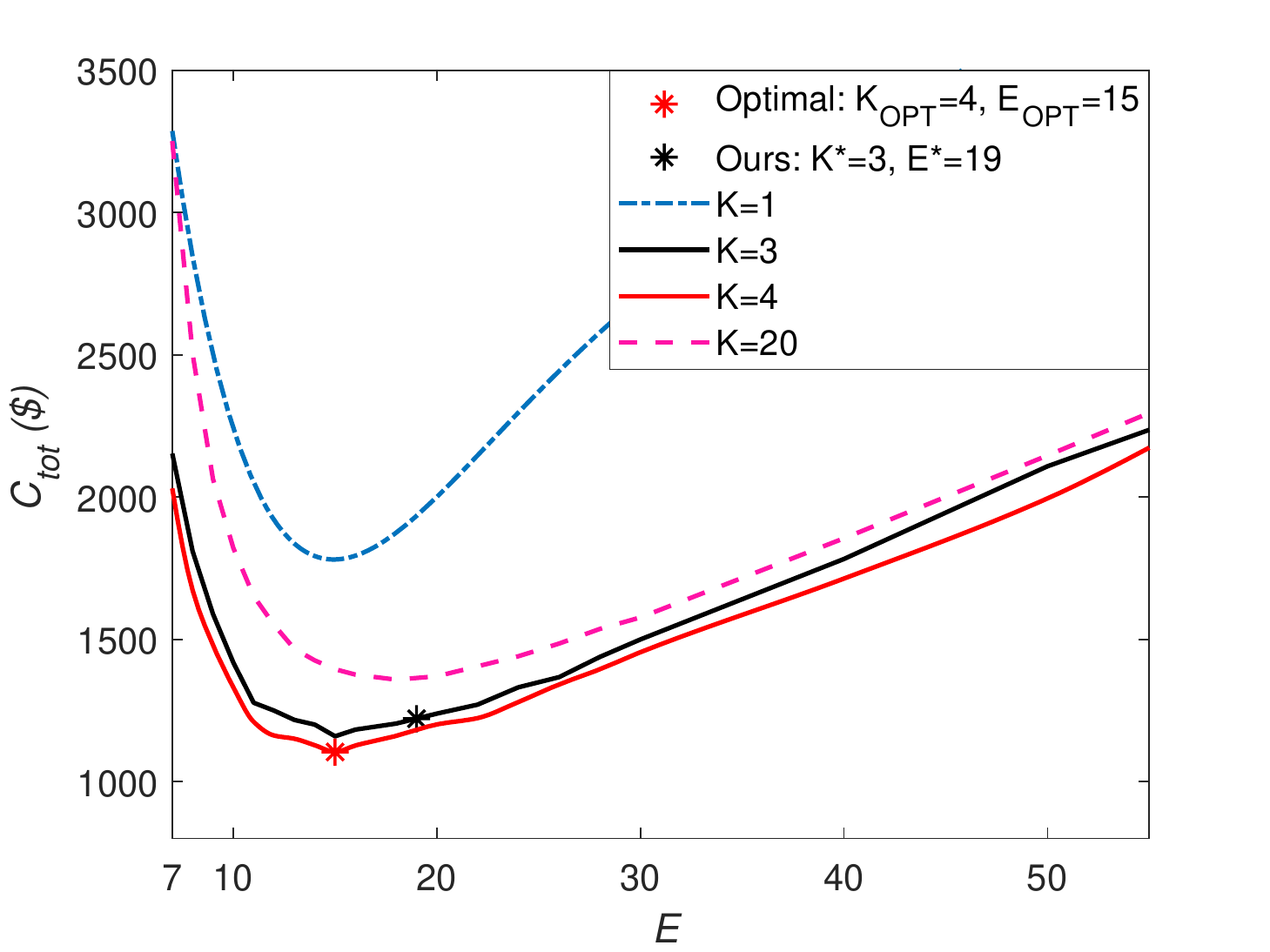}}
\subfigure[\label{ctot_gamma0}$C_\textnormal{tot}$ with $\left(K, E\right)$ for $\gamma=0$]{
\includegraphics[width=5.27cm,height=4.4cm]{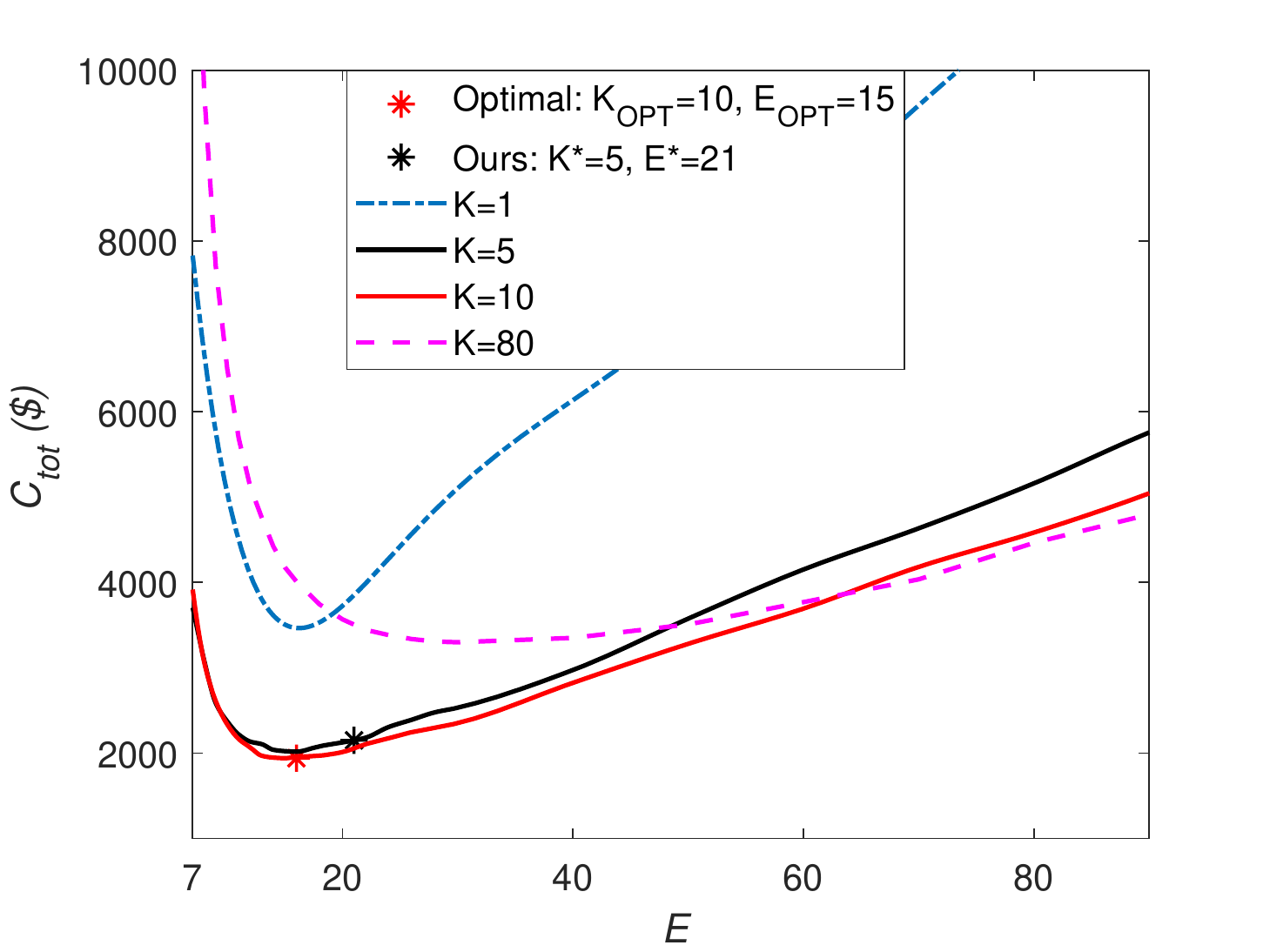}}
\subfigure[\label{ctot_gamma1}$C_\textnormal{tot}$ with $\left(K, E\right)$ for $\gamma=1$]{
\includegraphics[width=5.27cm,height=4.4cm]{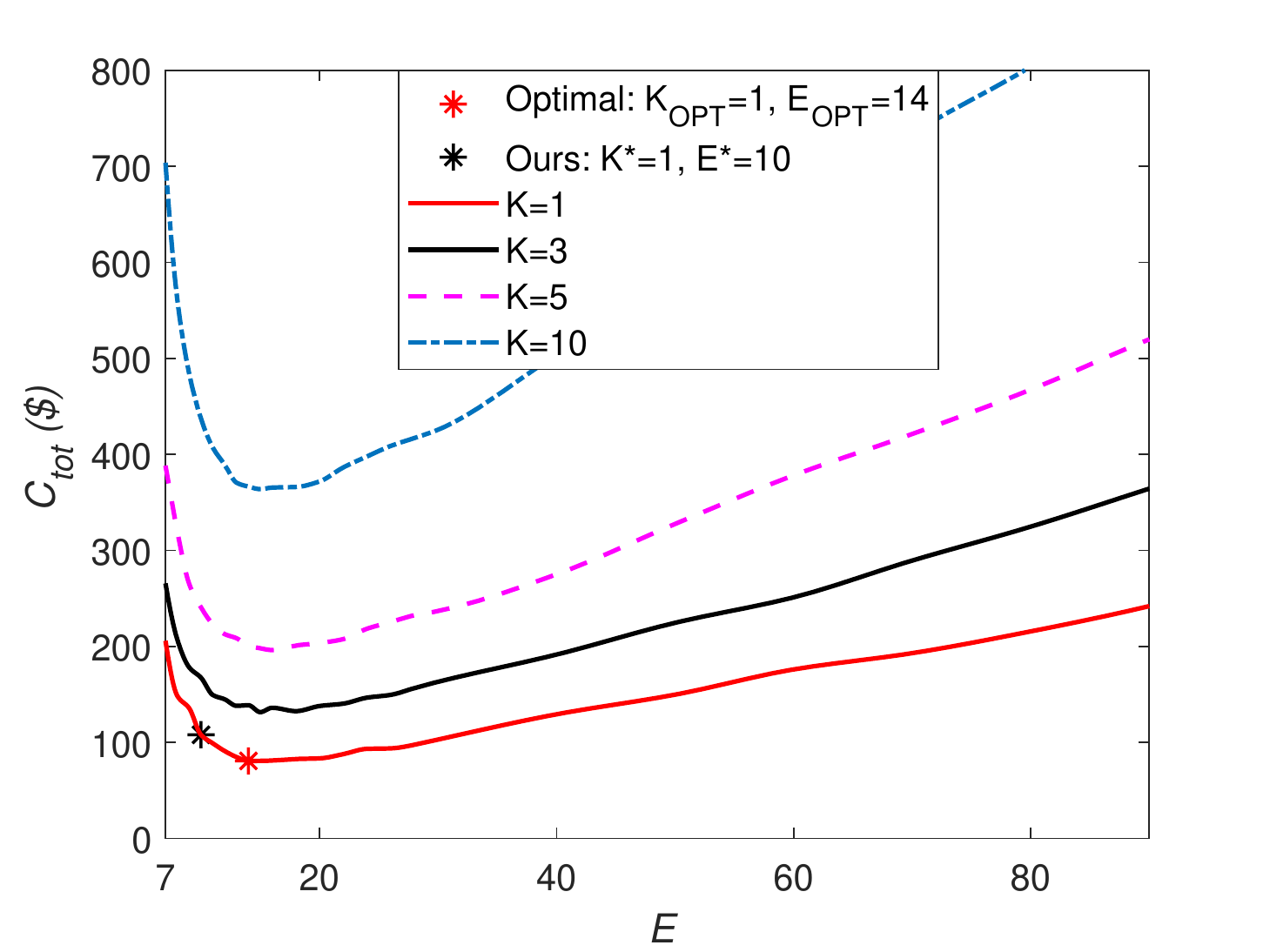}}
\subfigure[\label{tradeoffctot}$C_\textnormal{tot}$ with $\gamma$]{\includegraphics[width=5.3cm,height=4.41cm]{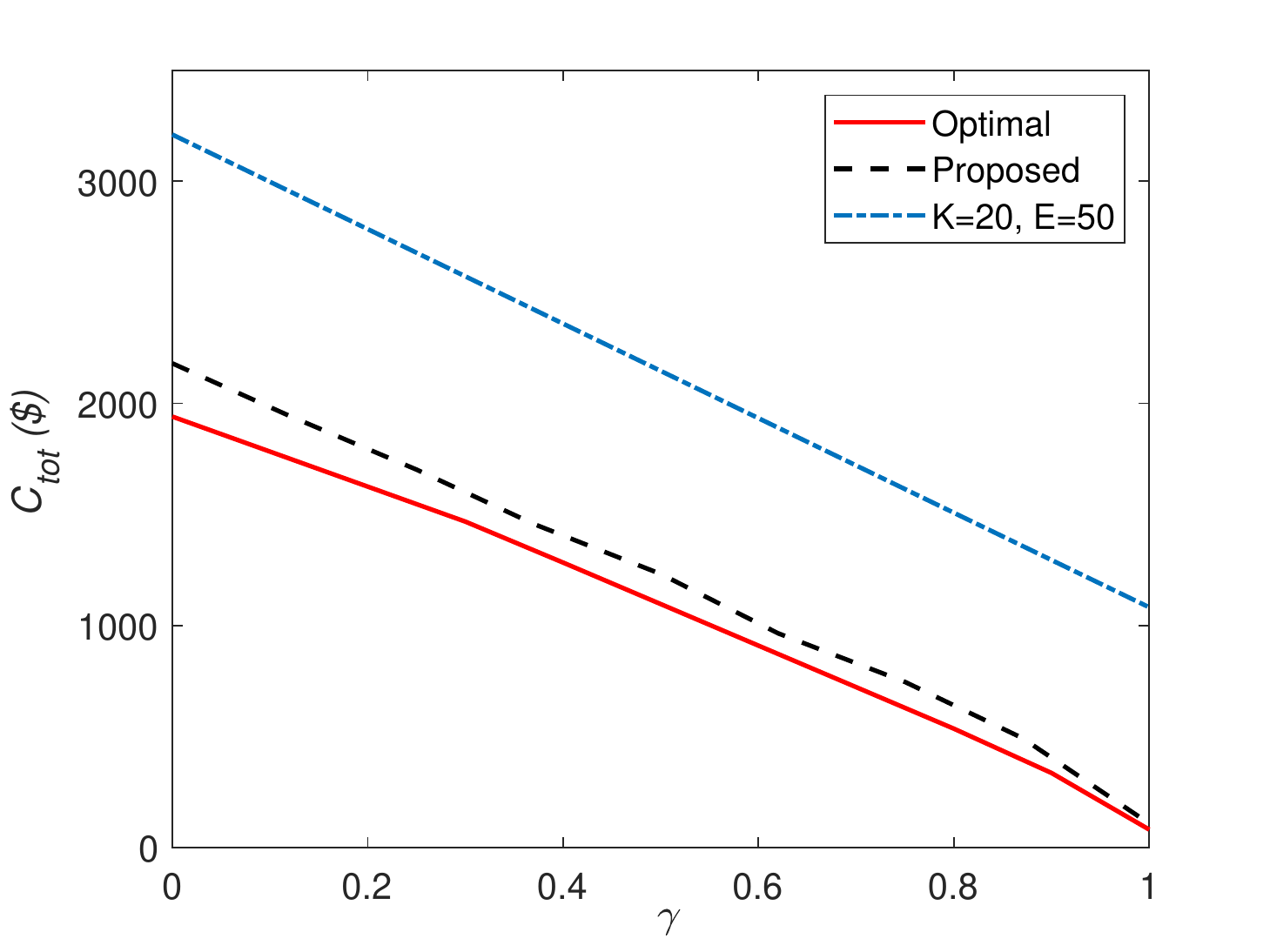}}
\caption{Performance of
$C_\textnormal{tot}$ for reaching the target loss 1.05  
for different $\gamma$ for 
\textbf{Simulation Setup 2} with logistic regression and Synthetic (1,1).  
(a)-(b): When 
$\gamma=0.5$
    our solution reaches the target loss $1.05$ using \$1221 compared to the empirical optimal \$1104. 
      (c)-(f): When $\gamma$ increases from 0 to 1, our proposed solutions reach the target loss using the similar total cost as the corresponding empirical optimal ones. 
}
\label{c_tot}
\end{figure*}

\subsubsection{Property Validation} 
\label{sec:propertyValidation}
We highlight that \textbf{our derived theoretical properties of $K$ and $E$ 
can be validated empirically} in both prototype and simulation experiments. 
Particularly,  we manually decrease $t_p$ and $e_p$ in Simulation Setup 2, and  
 shows how system parameters affect the design principles in  Fig.~\ref{property}.  
 We summarize the key results as follows.

  \begin{itemize}
  
   \item We observe from Figs.~\ref{ctot_gamma0}-\ref{ctot_gamma1} that the optimal $K$ decreases from 10 to 1 as $\gamma$ increases from 0 to 1. Particularly,  we see from Figs.~\ref{setup_femnist2} and \ref{ctot_gamma1} that, when $\gamma=1$, the energy cost  
    strictly increases in $K$, with $K_\textnormal{OPT}=1$. 
  These observations confirm our claim in Theorem~\ref{theorem:t_tot}.
  
 \item   Figs.~\ref{hd_t_tot}, \ref{setup_femnist2}, \ref{ctot_gamma05}--\ref{ctot_gamma1} demonstrate that 
  for any fixed value of  $K$,  the corresponding total cost first decreases and then increases as $E$ increases, which confirms our Theorem~\ref{corollary:t_tot_E1}.

 \item Comparing  Fig.~\ref{change_tp} with  Fig.~\ref{ctot_gamma0}, 
 we observe that for a fixed value of $E$, the optimal $K$ decreases as $t_p$ decreases. For example, for $E=26$, the optimal $K$ 
 decreases from $10$ to $2$ as $t_p$ decreases, 
 which confirms Theorem~\ref{corollary:t_tot_K}. 
 
 \item  Comparing  Fig.~\ref{change_ep} with  Fig.~\ref{ctot_gamma1}, 
 for fixed value of $K$, the optimal $E$ increases as $e_p$ decreases, e.g., for $K\!=\!1$, the optimal $E$  
   increases from 14 to 20, which confirms Theorem~\ref{corollary:t_tot_E2}.

\end{itemize}

\section{Conclusion}
\label{sec:conclusion}

In this work, we have studied the cost-effective design for federated learning in mobile edge networks. 
We proposed a new time-sharing scheduling scheme which captures system heterogeneity in terms of computation and wireless communication, and analyzed how to optimally choose the  number  of participating  clients  ($K$)  and  the  number  of  local  iterations ($E$), which are two essential control variables in FL, to minimize the total cost  
while ensuring convergence. 
We proposed a sampling-based control algorithm which efficiently solves the optimization problem with marginal overhead. We also derived 
insightful solution properties which helps identify the design principles for different optimization goals, e.g., reducing learning time or saving energy. 
Extensive experimentation results validated our theoretical analysis and demonstrated the effectiveness and efficiency of our control algorithm. 
Our optimization design is orthogonal to most works on resource allocation for FL systems, e.g, transmission power or CPU frequency, 
and can be used together with those techniques to further reduce the cost. 

\begin{figure*}[t]
\centering
\subfigure[\label{change_tp} $C_\textnormal{tot}$ with $\left(K, E\right)$ for $\gamma=0$ ]{\includegraphics[width=5.4cm,height=4.47cm]{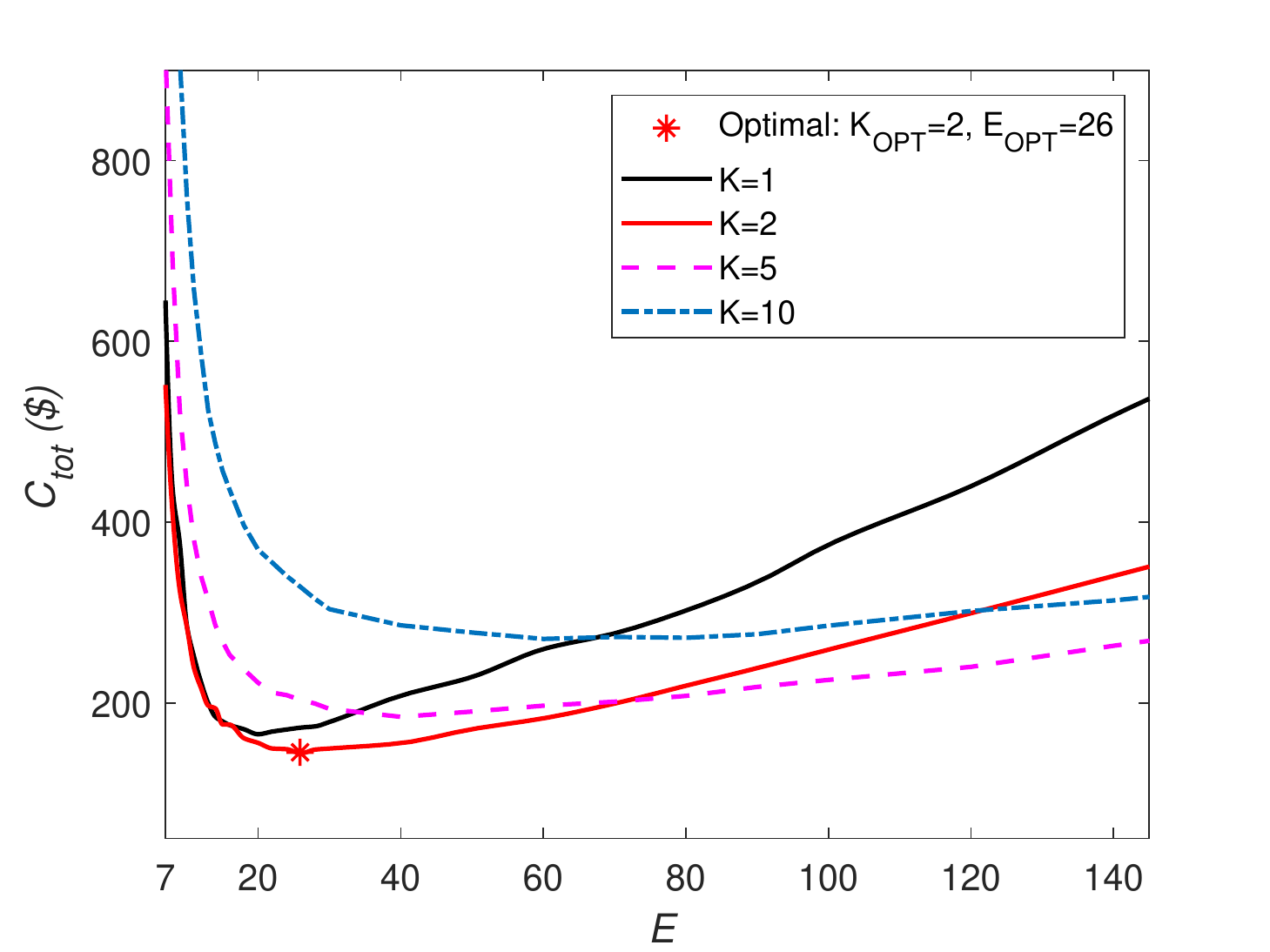}}
\subfigure[\label{change_ep} $C_\textnormal{tot}$ with $\left(K, E\right)$ for $\gamma=1$ ]{\includegraphics[width=5.4cm,height=4.47cm]{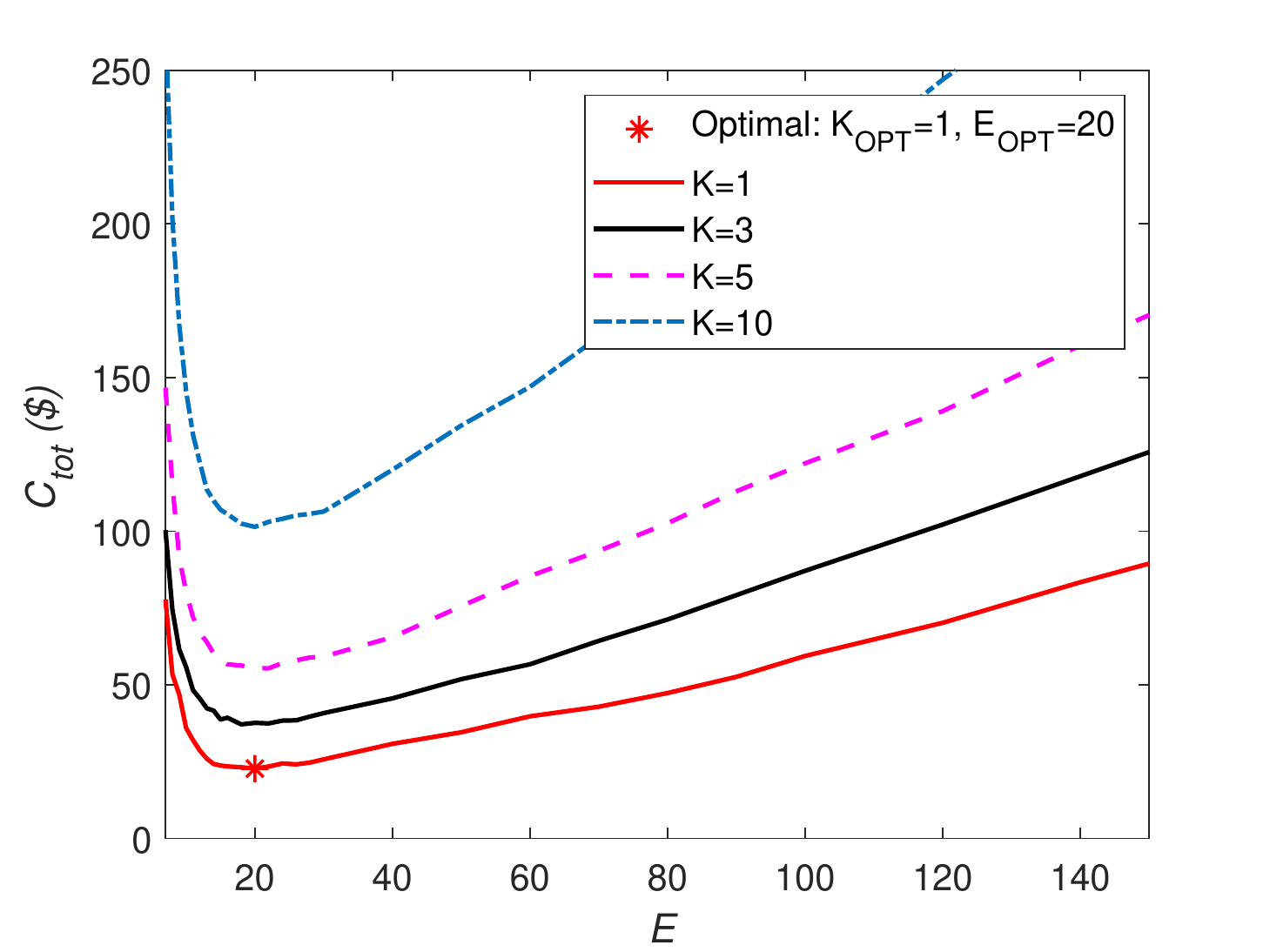}}
\caption{Performance of
$C_\textnormal{tot}$ for reaching the target loss 1.05  
for (a) $\gamma=0$ and (b) $\gamma=1$ for 
 \textbf{revised Simulation Setup 2}, where we manually decrease $t_p$ from $0.5$~s to $0.1$~s and  $e_p$ from $0.01$~J to $0.002$~J. 
}
\label{property}
\end{figure*}

 \appendices
\section {Proof of Theorem 1}
\label{theorem1app}
We prove by contradiction. For any reordered clients sequence $i\!-\!2$, $i\!-\!1$ and $i$ with  $t_{i-2,p}>t_{i-1,p}>t_{i,p}$, then the time after scheduling the three clients is:
\begin{equation}
\label{app_theo1}
\begin{aligned}
      T_i=&\max \left\{ {{t_{i,p}}, {T_{i - 1}}} \right\} + {t_{i,m}}\\
      =&\max \left\{ {{t_{i,p}}, {\max \left\{ {{t_{i-1,p}},{T_{i - 2}}} \right\} + {t_{i-1,m}}}} \right\} + {t_{i,m}}.
\end{aligned}
\end{equation}
Now, suppose we switch the scheduling order of $i\!-\!1$ and $i$, then the time after scheduling the three clients is:
\begin{equation}
\label{app_theo2}
\begin{aligned}
      T_{i-1}^\prime
      =\max \left\{ {{t _{i-1,p}}, {\max \left\{ {{t _{i,p}},{T_{i-2}}} \right\} + {t_{i,m}}}} \right\} + {t_{i-1,m}}.
\end{aligned}
\end{equation}
Case 1: if $t_{i,p}>t_{i-1,p}>T_{i-2}$, then $T_i$ in \eqref{app_theo1} reduces to
\begin{equation}
\begin{aligned}
T_i
=&\max \left\{ {{t _{i,p}}, \ {{t _{i-1,p}} \! +\! {t_{i-1,m}}}} \right\} + {t_{i,m}},
\end{aligned}
\end{equation}
and $T_{i-1}^\prime$ in \eqref{app_theo2} reduces to
\begin{equation}
\begin{aligned}
T_{i-1}^\prime
=&\max \left\{ {{t _{i-1,p}}, {{t _{i,p}} + {t_{i,m}}}} \right\} + {t_{i-1,m}}\\
=& {{t_{i,p}} + {t_{i,m}}}  + {t_{i-1,m}}\\
>& \max \left\{ {{t _{i,p}}, \ {{t _{i-1,p}} \! +\! {t_{i-1,m}}}} \right\} + {t_{i,m}}=T_i.
\end{aligned}
\end{equation}
Case 2: if $T_{i-2}>t_{i,p}>t_{i-1,p}$, then $T_i$ in \eqref{app_theo1} reduces to
\begin{equation}
\begin{aligned}
T_i
=&\max \left\{ {{t _{i,p}}, \ T_{i - 2} \! +\! {t_{i-1,m}}} \right\} + {t_{i,m}}\\
=&T_{i - 2} \! +\! {t_{i-1,m}} + {t_{i,m}},
\end{aligned}
\end{equation}
and $T_{i-1}^\prime$ in \eqref{app_theo2} reduces to
\begin{equation}
\begin{aligned}
T_{i-1}^\prime
=&\max \left\{ {{t _{i-1,p}}, {{T_{i - 2}} + {t_{i,m}}}} \right\} + {t_{i-1,m}}\\
=&T_{i - 2} \! +\! {t_{i-1,m}} + {t_{i,m}}=T_i.
\end{aligned}
\end{equation}
Case 3: if $t_{i,p}>T_{i-2}>t_{i-1,p}$, then $T_i$ in \eqref{app_theo1} reduces to
\begin{equation}
\begin{aligned}
T_i
=\max \left\{ {{t _{i,p}}, \ {{T _{i-2}}  + {t_{i-1,m}}}} \right\} + {t_{i,m}},
\end{aligned}
\end{equation}
and $T_{i-1}^\prime$ in \eqref{app_theo2} reduces to
\begin{equation}
\begin{aligned}
T_{i-1}
=&\max \left\{ {{t _{i-1,p}}, {{t _{i,p}} + {t_{i,m}}}} \right\} + {t_{i-1,m}}\\
=& {{t _{i,p}} + {t_{i,m}}}  + {t_{i-1,m}}\\
>& \max \left\{ {{t_{i,p}}, \ {{T_{i-2}}  + {t_{i-1,m}}}} \right\} + {t_{i,m}}=T_i.
\end{aligned}
\end{equation}
Therefore, we conclude that if any clients are not scheduled based on increasing order of $t_{k,p}$, scheduling the same sampled clients would result in a longer time.

\section{Proof of Lemma 2}
\label{lemma2app}
Since the $K$ clients are sampled uniformly at random in each  round \emph{without replacement}, the ``first" scheduling client is $\min_{k \in \mathcal{K}^{r}}\{t_{k,p}\}$.    
Thus, in the reordered sequence \eqref{reorder}, those who could be served as the first are the $\{1,2,\ldots,N\!-\!K\!+\!1\}$th clients.  

The probability of client $i\in\{1,2,\ldots,N-K+1\}$ being the ``first" client is $\frac{{C_{N-i}^{K-1}}}{C_{N}^{K}}$, where $C_{N}^{K}$ is the number of all possible $K$ clients combinations,  ${C_{N-k}^{K-1}}$ is the number of $K$ clients combinations with client $i$ being the ``first" (fast) client. In other words, the rest $K-1$ clients must be taken from the behind the $N-K$ clients who placed behind client $k$. 
Therefore, the expected time of the ``first" client is 
\begin{equation}
    \Expect[\min_{k \in \mathcal{K}^{r}}\{t_{k,p}E\}]=\frac{\sum\nolimits_{i=1}^{N-K+1}{C_{N-i}^{K-1}}t_{i,p} }{C_{N}^{K}}E.
\end{equation}

Then, similar to the proof in \emph{Lemma} 1, for uniform at random sampling, the probability of each client being sampled in each round is $\frac{K}{N}$. Thus. we have
\begin{equation}
\begin{aligned}
    \Expect[\sum_{k=1}^{K}t_{k,m}^{r}]=&\frac{K}{N} \Expect[\sum_{i=1}^{N}t_{i,m}^{r}]\\
    =&K\frac{\sum_{i=1}^{N}\Expect[t_{i,m}^{r}]}{N}
    =K\frac{\sum_{i=i}^{N}\overline{t}_{i,m}}{N}
    =t_mK.
\end{aligned}
\end{equation}

Given that the computation time and communication time are independent, the approximate per-round time in \eqref{app_tr_wireless} can be expressed as 
 \begin{equation}
 \begin{aligned}
\Expect[T^{r}] =&\Expect[\min_{k \in \mathcal{K}^{r}}\{t_{k,p}E\}]+ \Expect[\sum_{k=1}^{K}t_{k,m}^{r}]\\
=&\frac{\sum\nolimits_{i=1}^{N-K+1}{C_{N-i}^{K-1}}t_{i,p} }{C_{N}^{K}}E+t_mK.
 \end{aligned}
\end{equation}
Therefore,  \eqref{avgtime} can be obtained for $R$ round communications, which concludes this proof.

\section{Proof of Theorem 3}
\label{theorem3app}
Taking the first order derivative of $\tilde{\Expect}[C_\textnormal{tot}]$ over $K$ for any given $E$, we have
\begin{equation}
\label{partialK0}
\begin{aligned}
\frac{\partial\tilde{\Expect}[C_\textnormal{tot}]}{\partial K}=&\left[(1\!-\!\gamma)t_m+\gamma\left(e_pE \!+\!{e_m}\right)\right]\left(\frac{A_0}{E}+\frac{B_0(N\!-\!2)E}{N\!-\!1}\right)\\
&-\frac{(1\!-\!\gamma)B_0Nt_pE^2}{(N\!-\!1)K^2}.
\end{aligned}
\end{equation}
When $\gamma=1$, $\frac{\partial\tilde{\Expect}[C_\textnormal{tot}]}{\partial K}$ is always positive, thus, $K^\ast=1$. However, when $0\le \gamma<1$,  $\frac{\partial\tilde{\Expect}[C_\textnormal{tot}]}{\partial K}$ is negative for small $K$ and positive for large $K$. Thus, as $K$ increases, $\tilde{\Expect}[C_\textnormal{tot}]$ first decreases and then increases, thus $K^\ast \in [1,N]$.       
By letting $\frac{\partial\tilde{\Expect}[C_\textnormal{tot}]}{\partial K}=0$, 
and dividing $(1-\gamma)$, 
we concludes that $K^{\ast}$ decreases as $\gamma$ due to the fact that $\frac{\gamma}{(1-\gamma)}$ is an increasing function in $\gamma$.

\section{Proof of Theorem 4}
\label{theorem4app}
Following the proof in Appendix \ref{theorem3app}, we let $\frac{\partial\tilde{\Expect}[C_\textnormal{tot}]}{\partial K}=0$. Then, it is straightforward to obtain that, for any given $E$ and $0<\gamma<1$, the optimal $K$ increases in $t_p$ while decreases in  $t_m$, $e_m$, and $e_p$. In particular, according to Theorem \ref{theorem:t_tot}, when $\gamma=1$, $K^\ast=1$ and thus is independent with $e_m$ and $e_p$.

\section{Proof of Theorem 5}
\label{theorem5app}
Taking the first order derivative of $\tilde{\Expect}[C_\textnormal{tot}]$ over $E$ for any given $K$, we have
\begin{equation}
\begin{aligned}
\label{partialE_theorem4}
\frac{\partial\tilde{\Expect}[C_\textnormal{tot}]}{\partial E}&=B_{0}\left[(1\!-\!\gamma)\left( 2t_{p} E \!+\! t_{m}K \right)\!+\!\gamma K(2e_pE\!+\!e_m)\right] \\
&\cdot
\left(1\!+\!\frac{N\!-\!K}{K(N\!-\!1)}\!\right)-\frac{ A_{0}K[(1\!-\!\gamma)t_m\!+\!\gamma e_m]}{E^{2}}.
\end{aligned}\end{equation}
For any $0\le\gamma\le1$, and any feasible $K$, \eqref{partialE_theorem4} is negative when $E$ is small and positive when $E$ is large. Thus, as $E$ increases, $\tilde{\Expect}[C_\textnormal{tot}]$ first decreases and then increases. %

\section{Proof of Theorem 6}
\label{theorem6app}
Following the proof in Appendix \ref{theorem5app}, we let  $\frac{\partial\tilde{\Expect}[C_\textnormal{tot}]}{\partial E}=0$. Then, we have
\begin{equation}
\begin{aligned}
\label{partialE_appF}
\frac{\left[(1\!-\!\gamma)\left( 2t_{p} E^3 \!+\! t_{m}KE^2 \right)\!+\!\gamma K(2e_pE^3\!+\!e_mE^2)\right]}{K[(1\!-\!\gamma)t_m\!+\!\gamma e_m]} \\
=\frac{ A_{0}}{B_0\left(1\!+\!\frac{N\!-\!K}{K(N\!-\!1)}\!\right)}.
\end{aligned}
\end{equation}
Thus, it is straightforward to see that the solution of $E$ increases as $t_p$ or $e_p$ decreases. 

In particular, when $\gamma=0$, the solution of $E$ increases in $\frac{t_m}{t_p}$. 
Similarly,  when $\gamma=1$, the solution of $E$ increases in $\frac{e_m}{e_p}$.


\bibliographystyle{IEEEtran}
\bibliography{ref}
\end{document}